\documentclass{article} % For LaTeX2e
\usepackage{iclr2026_conference,times}

\usepackage{amsmath,amsfonts,bm}

\def\eqref#1{equation~\ref{#1}}
\def\1{\bm{1}}

\DeclareMathAlphabet{\mathsfit}{\encodingdefault}{\sfdefault}{m}{sl}
\SetMathAlphabet{\mathsfit}{bold}{\encodingdefault}{\sfdefault}{bx}{n}

\iclrfinalcopy

\usepackage{amssymb}
\usepackage{amsmath}
\usepackage{amsthm}
\usepackage{graphicx}
\usepackage{subcaption}
\usepackage{mathtools} % for \mathclap

\usepackage{thmtools} % for numbering theorems

\usepackage{hyperref}
\usepackage{url}
\usepackage[capitalise]{cleveref}
\usepackage{xspace}
\usepackage[show]{chato-notes}
\everypar\expandafter{\the\everypar\looseness=-1}
\usepackage{tikz}
\usepackage{pgfplots}
\pgfplotsset{compat=newest}
\usepackage{amsmath}
\usepackage{bm}
\usepackage{booktabs,makecell,xcolor,colortbl}
\definecolor{First}{RGB}{0,0,0}
\definecolor{Second}{RGB}{0,0,0}
\definecolor{Third}{RGB}{0,0,0}
\newcommand{\first}[1]{\textcolor{First}{\bm{#1}}}
\newcommand{\second}[1]{\textcolor{Second}{\underline{#1}}}

\newcommand{\ii}{i}
\newcommand{\Lc}{\textbf{L}^{\vec{\mathcal{F}}}}
\newcommand{\Qc}{\textbf{Q}^{\vec{\mathcal{F}}}}
\newcommand{\Dvc}{\mathbf{D}_V}

\theoremstyle{definition}
\newtheorem{definition}{Definition}

\newtheorem{lemma}{Lemma}

\newcommand{\approach}{DSHN\xspace}
\newcommand{\light}{DSHNLight\xspace}
\newcommand{\ndatasetsgood}{6\xspace}
\newcommand{\ndatasetstotal}{7\xspace}
\newcommand{\nbaselines}{13\xspace}

\title{Directional Sheaf Hypergraph Networks: Unifying Learning on Directed and Undirected Hypergraphs}

\author{\\
\parbox{\textwidth}{
{\textbf{Emanuele Mule}$^{1{*}}$
}
\hspace{0.32em}
{\textbf{Stefano Fiorini}$^{2}$}
\hspace{0.32em}
{\textbf{Antonio Purificato}$^{1,3{\dagger}}$}
\hspace{0.36em} %more space due to more notes
{\textbf{Federico Siciliano}$^{1}$} \\
\hspace{0.32em}
{\textbf{Stefano Coniglio}$^{4}$}
\hspace{0.32em}
{\textbf{Fabrizio Silvestri}$^{1}$} \\[0.6em]
\normalfont
$^{1}$ Sapienza University of Rome, Rome, Italy \\
$^{2}$ Independent Researcher \\
$^{3}$ Amazon Research \\
$^{4}$ University of Bergamo, Bergamo, Italy
}}
\colorlet{col4}{blue}
\begin{document}

\maketitle
\renewcommand{\thefootnote}{\fnsymbol{footnote}}
\footnotetext[1]{Corresponding author: \texttt{emi.mule2001@gmail.com}.}
\footnotetext[2]{Work done outside of the company.}
\renewcommand{\thefootnote}{\arabic{footnote}}

\begin{abstract}
Hypergraphs provide a natural way to represent higher-order interactions among multiple entities. While undirected hypergraphs have been extensively studied, the case of directed hypergraphs, which can model oriented group interactions, remains largely under-explored despite its relevance for many applications. Recent approaches in this direction often exhibit an implicit bias toward homophily, which limits their effectiveness in heterophilic settings.
Rooted in the algebraic topology notion of Cellular Sheaves, Sheaf Neural Networks (SNNs) were introduced as an effective solution to circumvent such a drawback. While a generalization to hypergraphs is known, it is only suitable for undirected hypergraphs, failing to tackle the directed case.
In this work, we introduce \emph{Directional Sheaf Hypergraph Networks} (\emph{\approach}), a framework integrating sheaf theory with a principled treatment of asymmetric relations within a hypergraph. From it, we construct the \emph{Directed Sheaf Hypergraph Laplacian}, a complex-valued operator by which we unify and generalize many existing Laplacian matrices proposed in the graph- and hypergraph-learning literature.
Across 7 real-world datasets and against \nbaselines baselines, \approach achieves relative accuracy gains from 2\% up to 20\%, showing how a principled treatment of directionality in hypergraphs, combined with the expressive power of sheaves, can substantially improve performance.
\end{abstract}

\section{Introduction}\label{sec:introduction}
Learning from structured, non-Euclidean data has been dominated by Graph Neural Networks (GNNs), which propagate and aggregate features along pairwise edges~\citep{GNN, noneuclidean}.
Sheaf Neural Networks (SNNs)~\citep{hansen2020sheafneuralnetworks, bodnar2023neuralsheafdiffusiontopological} extend GNNs by leveraging the algebraic concept of a \emph{cellular sheaf}. They assign vector spaces to nodes and edges, along with learnable \emph{restriction maps} propagating information between them. By operating in higher-dimensional feature spaces, SNNs mitigate oversmoothing and enhance performance on heterophilic graphs, where neighboring nodes exhibit dissimilar features~\citep{purificato2025sheaf4rec}.

While effective on graphs, both GNNs and SNNs are inherently limited to dyadic relations. Many real-world systems such as social networks~\citep{Benson_2016, Benson_2018}, biological systems~\citep{traversa2023robustnesscomplexitydirectedweighted}, and protein interactions~\citep{murgas2022hypergraph} exhibit multi-way relationships that cannot be captured by pairwise links alone. Hypergraph Neural Networks (HGNNs) address this limitation by modeling hyperedges as sets of nodes, enabling the learning of multi-entity dependencies~\citep{murgas2022hypergraph,chen2022preventingoversmoothinghypergraphneural}. 
However, traditional HGNNs face two main limitations. First, they inherit fundamental drawbacks from their graph-based counterparts: many architectures assume homophily, which is often violated in heterophilic settings, and are prone to oversmoothing, where deep message passing causes node representations to converge and lose discriminative power~\citep{heterophilyhypergraphs,telyatnikov2025hypergraph,chen2022preventingoversmoothinghypergraphneural,nguyen2023revisitingoversmoothingoversquashingusing}. Second, most HGNNs are formulated for undirected hypergraphs, treating hyperedges symmetrically and neglecting orientation even when such hyperedges encode asymmetric or causal relationships such as chemical reactions, metabolic pathways, and causal multi-agent interactions~\citep{D2RE00309K,traversa2023robustnesscomplexitydirectedweighted}.

Sheaf Hypergraph Networks (SHNs)~\citep{duta} address the first of these challenges by extending the principles of SNNs to hypergraphs. They assign vector spaces to nodes and hyperedges and propagate information via learnable restriction maps
%, enforcing global consistency across the network. This design
naturally mitigating the issues associated with oversmoothing and heterophily while generalizing message passing to higher-order, multi-way interactions, which provides a more expressive framework than traditional HGNNs.

Despite these advantages, SHNs have two key limitations.
$(i)$ They only model hyperedges as undirected, limiting their ability to capture asymmetric directional relationships. Indeed, while directionality has been recently incorporated in GNNs via, e.g., specialized Laplacians such as the one by
%, including PageRank-based approximations
\cite{pagerankdirectedgraph} and complex-valued operators~\citep{zhang2021magnetneuralnetworkdirected,fiorini2023sigmanetlaplacianrule}, extensions to hypergraphs remain limited~\citep{fiorini2024let} and often task-specific~\citep{Gatta, dhmconv}. To our knowledge, no SHN methods which can handle directed hypergraphs are known.
%This motivates the need for a framework that simultaneously models higher-order interactions and preserves hyperedge directionality.
%
% which are essential in many real-world systems such as chemical reactions, metabolic pathways, and causal multi-agent interactions~\citep{D2RE00309K,traversa2023robustnesscomplexitydirectedweighted}. 
%
$ii)$ Second (as we show in this paper), the
%linear
Laplacian operator proposed in~\cite{duta} fails to satisfy the spectral properties required of a well-defined convolutional operator, such as positive semidefiniteness, contrarily to what the authors report (and claim to have proven) in their paper.

%Solution
In this paper, we introduce \emph{Directional Sheaf Hypergraph Networks} (\emph{\approach}), a principled extension of SHNs to directed hypergraphs. Specifically, we define \emph{Directed Hypergraph Cellular Sheaves}, equipping hyperedges not only with the notion of tail and head sets (source and target nodes, respectively) but also with \emph{asymmetric} restriction maps that respect orientation within a hyperedge. From this, we derive the \emph{Directed Sheaf Hypergraph Laplacian}, a novel complex-valued Hermitian operator whose phase naturally encodes direction while preserving essential spectral properties, including admitting a spectral decomposition with real-valued, nonnegative eigenvalues. 
We evaluate \approach on \ndatasetstotal real-world datasets, as well as on synthetic benchmarks specifically designed to test \approach's ability to capture directional information within a hypergraph. Compared to \nbaselines state-of-the-art baselines, our method achieves relative accuracy gains from 2\% up to 20\%, demonstrating that explicitly modeling orientation via our proposed asymmetric and complex-valued restriction maps improves predictive performance.
Our contributions can be summarized as follows:
\begin{itemize}
    \item We introduce the concept of \emph{Directed Hypergraph Cellular Sheaves}, a framework that extends directed hypergraphs by providing a principled representation of directional interactions. This is achieved by assigning complex-valued linear maps between nodes and hyperedges, capturing the node-to-hyperedge relationships within each directed hyperedge.
    \item We introduce the \emph{Directed Sheaf Hypergraph Laplacian}, a novel complex-valued Hermitian matrix that satisfies the key properties required of a well-defined spectral operator. Our formulation generalizes existing graph and hypergraph Laplacians, providing a unified framework for learning on hypergraphs with both directed and undirected hyperedges.
    \item We introduce \emph{Directional Sheaf Hypergraph Networks} (\approach), a model that combines Sheaf theory with a principled treatment of directional information, enabling state-of-the-art performance on directed hypergraph benchmarks.\footnote{We provide the code to reproduce the results at \url{https://github.com/EmaMule/DirectionalSheafHypergraphs}.}
\end{itemize}

\section{Background and Previous Work}
\paragraph{Sheaf Neural Networks and Sheaf Hypergraph Networks.}

Sheaf theory provides a principled framework for modeling local information flow across structured domains. In algebraic topology, a sheaf associates data to open sets together with restriction maps ensuring consistency~\citep{curry2014sheavescosheavesapplications}. Cellular sheaves adapt this idea to cell complexes by assigning vector spaces to cells with linear maps along face relations. Building on this, Sheaf Neural Networks (SNNs)~\citep{hansen2020sheafneuralnetworks, bodnar2023neuralsheafdiffusiontopological} assign vector spaces to graph nodes and edges and learn restriction maps for each node-edge incidence relationship, generalizing message passing. SNNs are particularly effective in heterophilic settings and help mitigate oversmoothing, a common drawback of deep GNNs. 
As described by~\cite{discourse}, attaching a cellular sheaf to a graph can be interpreted through the lens of opinion dynamics. Unlike traditional Graph Neural Networks (GNNs), each node's representation ``lives'' in its own vector space, representing a private opinion, while the restriction maps, linear transformations between node and edge vector spaces, govern how this opinion is expressed along each incident edge. In this way, a node can maintain its own distinct representation while still having different ``opinions'' across the various edges they participate in, allowing for more expressive and flexible representations compared to standard spectral-based methods, where adjacent nodes tend to have similar representations.

Expanding on this idea,~\cite{duta} extend SNNs to the hypergraph setting. However, as it will be shown in this work, their formulation suffers from two key limitations. First, it does not capture directionality in hypergraphs. Second, although the \emph{Sheaf Hypergraph Laplacian} is presented as a Laplacian operator, as shown in our paper, it fails to satisfy fundamental spectral properties expected of such operators, most notably positive semidefiniteness.

%
%More recently, the construction of a \emph{Sheaf Hypergraph Laplacian} has been revisited from the perspective of simplicial complexes~\citep{simplicial-sheaf}, further generalizing the framework of~\cite{hansen2020sheafneuralnetworks} to higher-order structures.
%
{\bf Laplacian matrices for Directed Graphs}\phantom{aa}
Classical spectral methods define convolutional operators through the graph Laplacian \citep{biggs1993algebraic,chebnet,kipf_semi-supervised_2017}. While effective, such methods require to either work on inherently undirected graphs or to symmetrize the graph's adjacency matrix, thereby discarding edge directionality. Drawing inspiration from the Magnetic Laplacian introduced by \citet{lieb1992fluxeslaplacianskasteleynstheorem} in the study of electromagnetic fields, spectral-based methods have been extended to incorporate edge directionality. In particular, \citet{zhang2021magnetneuralnetworkdirected,fiorini2023sigmanetlaplacianrule} developed operators that encode orientation in the imaginary part of complex-valued Hermitian matrices. This construction preserves the desirable spectral properties required for a well-defined convolutional operator while embedding directional information, enabling convolutional operators to faithfully capture the asymmetry of directed graphs.

{\bf Undirected and Directed Hypergraphs} \phantom{aa}
%A hypergraph is an ordered pair $\gH = (V, E)$, with $n := |V|$ and $m := |E|$, where $V$ is the set of vertices (or nodes) and $E \subseteq 2^{V} \setminus \{\}$ is the (nonempty) set of hyperedges.
%
The hyperedge weights are stored in the diagonal matrix $W \in \mathbb{R}^{m \times m}$.
The vertex and hyperedge degrees are defined as $\mathbf{D}_u = \sum_{e \in E: u \in e} |w_e|$ for $u \in V$ and $\delta_e = |e|$ for $e \in E$.
Hypergraphs where $\delta_e = k$ for some $k \in \mathbb{N}$ for all $e \in E$ are called $k$-uniform. Graphs are $2$-uniform hypergraphs.
Following~\cite{GALLO1993177}, we define a directed hypergraph as a hypergraph where each edge $e \in E$ is partitioned in a {\em tail set} $T(e)$ and a \textit{head set} $H(e)$. If $H(e)$ is empty, $e$ is an undirected edge.
Research on learning on directed hypergraphs remains limited, with most existing works either constrained to task-specific scenarios~\citep{Luo, Gatta} or restricted to 2-uniform directed hypergraphs~\citep{dhmconv,DHGNN}.
This gap has been recently addressed through the Generalized Directed Laplacian~\citep{fiorini2024let}, a complex Hermitian operator that unifies directed and undirected hypergraphs, and extends several popular methods for directed graphs to the hypergraph domain. However, their method, while being suitable for directed hypergraphs, is still implicitly biased towards homophilic settings and can be prone to oversmoothing.
%
% \blue{
% \paragraph{Notation}
% Throughout the paper, we use uppercase bold letters to denote matrices and lowercase bold letters to denote vectors. We denote $\mathcal{F}$ the cellular sheaf associated with the hypergraph, where $\vec{\mathcal{F}}(u)$ for $u \in V$ and $\vec{\mathcal{F}}(e)$ for $e \in E$ denote the vector spaces attached to nodes and hyperedges, respectively. For a node $u$ and a hyperedge $e$, we write $u \trianglelefteq e$ to indicate incidence (i.e., $u \in \Gamma(e)$, where $\Gamma(e)$ denotes the set of nodes belonging to the hyperedge). The corresponding restriction maps are written as $\mathcal{F}_{u \trianglelefteq e}$; although these are matrices, we avoid boldface for readability and consistency with prior work.  
% %
% Arrow notation is used to indicate directionality: $\mathcal{F}_{u \trianglelefteq e}$ denotes a directionless restriction map, whereas $\vec{\mathcal{F}}_{u \trianglelefteq e}$ denotes a directed restriction map. This convention extends to all analogous objects in the paper. We use $\dagger$ to denote the conjugate transpose while we denote with $\|\cdot\|$ the Euclidean norm. Finally, we write $A \,\|\,B$ to indicate the concatenation of two objects of matching shape.
% }

\section{Directed Sheaf Hypergraph Laplacian}\label{sec:directed-sheaf-hypergraph-laplacian}

\subsection{Directed Hypergraph Cellular Sheaf}

In this work, we introduce the notion of \emph{Directed  Hypergraph Cellular Sheaf}, which assigns to a directed hypergraph complex-valued restriction maps designed to capture and encode directional information contained in the hypergraph’s underlying topology.

\begin{definition}
\label{def:directed-cellular-sheaf-hypergraph}
The Directed Hypergraph Cellular Sheaf of a directed hypergraph $\mathcal{H} = (V, E)$ is the tuple
\(
\bigl\langle
\mathbf{\mathcal{S}}^{(q)},
\{\vec{\mathcal{F}}(u)\}_{u\in V},
\{\vec{\mathcal{F}}(e)\}_{e\in E},
\{\vec{\mathcal{F}}_{u\trianglelefteq e}\}_{u\in \Gamma{(e)}}
\bigr\rangle,
\) consisting of: 
\begin{enumerate}
    \item A complex-valued matrix $\mathbf{\mathcal{S}}^{(q)} \in \mathbb{C}^{m \times n}$ with $q \in \mathbb{R}$, defined entry-wise for each hyperedge $e \in E $ and node $u \in V$ as:
    $$
    \mathbf{\mathcal{S}}^{(q)}_{u \trianglelefteq e} = 
    \begin{cases}
    1 & \text{if } u \in H(e) \quad \text{(head set)} \\
    e^{-2\pi \ii q} & \text{if } u \in T(e) \quad \text{(tail set)} \\
    0 & \text{otherwise}
    \end{cases}
    $$
    \item A vector space $\vec{\mathcal{F}}(u) \subseteq\mathbb{C}^d$ associated with each node $u \in V$;
    \item A vector space $\vec{\mathcal{F}}(e) \subseteq \mathbb{C}^d$ associated with each hyperedge $e \in E$;
    \item A restriction map $\vec{\mathcal{F}}_{u \trianglelefteq e}: \vec{\mathcal{F}}(u) \rightarrow \vec{\mathcal{F}}(e)$ with $\vec{\mathcal{F}}_{u \trianglelefteq e} = \mathbf{\mathcal{S}}^{(q)}_{u \trianglelefteq e}  \mathcal{F}_{u \trianglelefteq e}\in \mathbb{C}^{d \times d}$ 
    where $\mathcal{F}_{u \trianglelefteq e} \in \mathbb{R}^{d \times d}$ is a real-valued, directionless, restriction map.
\end{enumerate}
\end{definition} 

The idea is to associate to each node-hyperedge incidence relationship a linear restriction map $\vec{\mathcal{F}}_{v \trianglelefteq e}$ which can either be real- or complex-valued, based on the directional matrix $\mathbf{\mathcal{S}}^{(q)}$, specifying whether a node within a hyperedge belongs to the tail or to the head set. In line with \cite{zhang2021magnetneuralnetworkdirected}, although their work focuses on directed graphs and hence fails to model many-to-many interactions, the parameter $q$ associated with the matrix $\mathbf{\mathcal{S}}^{(q)}$ serves as a \emph{charge parameter} that controls the relevance of the hypergraph's directional information.
In fact, when $q=0$, the restriction maps are all real-valued independently of the hypergraph's directions, and we come back to the definition of a (Hypergraph) Cellular Sheaf as introduced in~\cite{duta}. 
% %
While prior work has incorporated directional information in hypergraphs using complex-valued coefficients, these approaches typically rely on a fixed complex phase~\citep{fiorini2024let}. In contrast, our formulation introduces a tunable complex-valued coefficient, allowing the model to flexibly adjust the contribution of directional information. Moreover, by equipping the topology of the hypergraph with a Cellular Sheaf, our method provides a more expressive framework for representing hypergraph structures than existing directed hypergraph neural networks.
\looseness -1 To visualize of a Directed Hypergraph Cellular Sheaf associated to a directed hyperedge, see \cref{fig:directed-sheaf-hypergraph}.

\begin{figure}[h]
    \centering
    \includegraphics[width=0.9\linewidth]{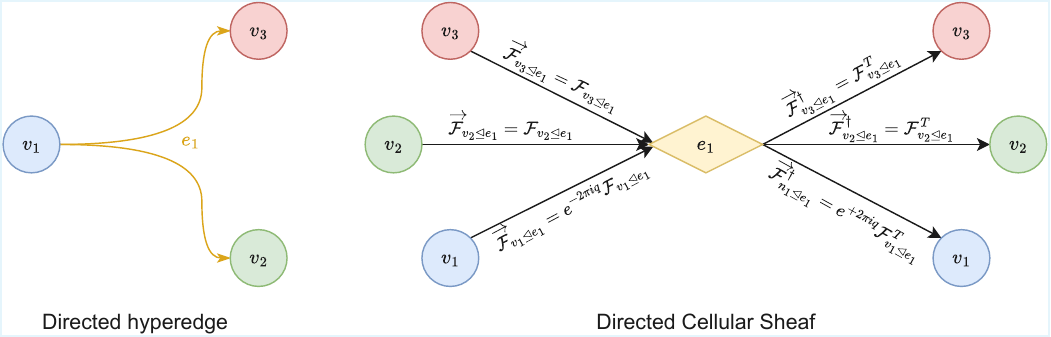}
    \caption{Visualization of sheaves over a directed hyperedge, illustrating the incidence relationship between nodes and the hyperedge, together with the restriction maps \( \vec{\mathcal{F}}_{v \trianglelefteq e} \). The tail node $v_1$ is encoded via the $e^{-2\pi i q}$ coefficient which pre-multiplies the directionless restriction map \( \mathcal{F}_{v \trianglelefteq e} \).}
    \label{fig:directed-sheaf-hypergraph}
\end{figure}

\subsection{Directed Sheaf Hypergraph Laplacian}

Given a directed hypergraph and its corresponding Directed Hypergraph Cellular Sheaf $\vec{\mathcal{F}}$, let $\mathbf{B}^{(q)} \in \mathbb{C}^{md\times nd}$ be a complex-valued incidence matrix which, for each pair $e \in E$ and $u \in V$, reads:
   \begin{equation*}
           \mathbf{B}^{(q)}_{e u} = 
    \begin{cases}
    \vec{\mathcal{F}}_{u \trianglelefteq e} = \mathbf{\mathcal{S}}^{(q)}_{u \trianglelefteq e}\mathcal{F}_{u \trianglelefteq e} & \text{if } u \in e\\
    0 & \text{otherwise}.
    \end{cases}
   \end{equation*}

When factoring in, for each pair $e \in E$ and $u \in e$, whether $u$ belongs to the head or tail set of $e$, we obtain:
    \begin{equation}\label{eq:incidence-expanded}
    \mathbf{B}^{(q)}_{e u} = 
    \begin{cases}
        \vec{\mathcal{F}}_{u \trianglelefteq e} = \mathbf{\mathcal{S}}^{(q)}_{u \trianglelefteq e} \mathcal{F}_{u \trianglelefteq e} = \mathcal{F}_{u \trianglelefteq e} & \text{if } u \in H(e) \quad \text{(head set)} \\
        \vec{\mathcal{F}}_{u \trianglelefteq e} = \mathbf{\mathcal{S}}^{(q)}_{u \trianglelefteq e} \mathcal{F}_{u \trianglelefteq e} =  e^{-2\pi \ii q} \mathcal{F}_{u \trianglelefteq e} & \text{if } u \in T(e) \quad \text{(tail set)} \\
        0 & \text{otherwise}.
    \end{cases}
    \end{equation}

We define the \emph{Directed Sheaf Hypergraph Laplacian} $\Lc$ associated with a Directed Hypergraph Cellular Sheaf as follows:
\begin{equation}\label{eq:laplaciano}
    \Lc :=  \Dvc - \Qc \qquad  \text{with} \qquad \Qc :=  {\mathbf{B}^{(q)}}^\dagger \mathbf{D}_E^{-1} \mathbf{B}^{(q)},
\end{equation}
where $\Qc$ is the \emph{Directed Sheaf Hypergraph Signless Laplacian}. 

In the formula, $\mathbf{D}_E$ is the block-diagonal hyperedge degree matrix $\mathbf{D}_E := \text{diag}(\delta_1 \textbf{I}_d,\dots,\delta_m\textbf{I}_d) \in \mathbb{R}^{md\times md}$, where $\delta_e := |e|$ is the degree of hyperedge $e \in E$, and $\Dvc \in \mathbb{R}^{nd \times nd}$ is the block-diagonal node degree matrix defined as
$\Dvc:= \mathrm{diag}\left(\textbf{D}_{u_1}, \textbf{D}_{u_2}, \dots, \textbf{D}_{u_n}\right),$ with $\textbf{D}_{u_i} := \sum_{e \in E: u \in e}\vec{\mathcal{F}}_{u \trianglelefteq e}^\dagger \, \vec{\mathcal{F}}_{u \trianglelefteq e} \in \mathbb{R}^{d \times d}, u_i \in V$.

To clarify how $\textbf{L}^{\vec{\mathcal F}}$ encodes the hypergraph structure, let us examine its entry corresponding to a pair of vertices $u,v \in V$:
\begin{equation} \label{eq:laplacian-expanded}
    (\mathbf{L}^{\vec{\mathcal{F}}})_{uv} = 
    \left\{
    \begin{array}{lr}
    \displaystyle
    \mathbf{D}_u - \sum_{e: u \in e} \frac{1}{\delta_{e}}{\mathcal{F}}_{u \trianglelefteq e}^{\top}{\mathcal{F}}_{u \trianglelefteq e}
    = 
    \sum_{e: u \in e} (1 - \frac{1}{\delta_{e}}){\mathcal{F}}_{u \trianglelefteq e}^{\top}{\mathcal{F}}_{u \trianglelefteq e} & u = v\\[5pt]
    \displaystyle
    - \sum_{e: u,v \in e} \frac{1}{\delta_{e}}  \vec{\mathcal{F}}_{u \trianglelefteq e}^{\dagger} \vec{\mathcal{F}}_{v \trianglelefteq e}  =
    - \sum_{e: u,v \in e} \frac{1}{\delta_{e}} (\mathbf{\mathcal{S}}^{(q)}_{u \trianglelefteq e})^\dagger (\mathbf{\mathcal{S}}^{(q)}_{v \trianglelefteq e}){\mathcal{F}}_{u \trianglelefteq e}^{\top} \mathcal{F}_{v \trianglelefteq e} & u \neq v.
    \end{array}
    \right.
\end{equation}

Note that the block-diagonal entries of $\mathbf{L}^{\vec{\mathcal{F}}}$ are always real. In contrast, its off-block-diagonal entries are complex-valued when the hypergraph is directed and $q \neq 0$, and real-valued if the hypergraph is undirected. By setting $q=0$, the hypergraph directions are entirely disregarded and $\mathbf{L}^{\vec{\mathcal{F}}}$ coincides with the Laplacian matrix of its undirected counterpart.
The off-diagonal products in $\textbf{L}^{\vec{\mathcal{F}}}$ strongly depend on the interaction between the two components of the directional matrix $\mathbf{\mathcal{S}}^{(q)}$ associated to the restriction maps. The product of these contributes to the real and imaginary parts of $\textbf{L}^{\vec{\mathcal{F}}}$ as follows:
\begin{equation}\label{eq:interaction-types}
(\mathbf{\mathcal{S}}^{(q)}_{v \trianglelefteq e})^\dagger \mathbf{\mathcal{S}}^{(q)}_{u \trianglelefteq e}
=
\begin{cases}
1 & v,u \in T(e) \ \text{(tail-tail)},\\
1 & v,u \in H(e) \ \text{(head-head)},\\
e^{+2\pi i q} & v \in T(e),\ u \in H(e) \ \text{(tail-head)} \quad \text{(equal to $i$ if $q=\frac{1}{4}$)},\\
e^{-2\pi i q} & v \in H(e),\ u \in T(e) \ \text{(head-tail)} \quad \text{(equal to $-i$ if $q=\frac{1}{4}$)}.
\end{cases}
\end{equation}

We can expand \cref{eq:laplacian-expanded} by considering the special case where $q=\frac{1}{4}$. In this case, each entry of $\textbf{L}^{\vec{\mathcal{F}}}$ can be expressed to explicitly highlight the impact of the directionality of each hyperedge:
\begin{eqnarray}\label{eq:laplacian-expanded-special}
(\mathbf{L}^{\vec{\mathcal{F}}})_{uv} =
\left\{
\begin{array}{ll}
\displaystyle
\sum_{e : u \in e}\!\left( 1 - \frac{1}{\delta_e}\right)
\mathcal{F}_{u \trianglelefteq e}^{\top}\mathcal{F}_{u \trianglelefteq e},
& u = v, \\
\displaystyle
-\!\!\!\!\!\!\!\sum_{\begin{subarray}{c}
e \in E \\
u,v \in H(e)\\
\lor  u,v \in T(e)
\end{subarray}}
\!\!\!\!\!
\frac{1}{\delta_e}\,\mathcal{F}_{u \trianglelefteq e}^{\top}\mathcal{F}_{v \trianglelefteq e}
-
%SC: occhoi -- questo in Gedi è un -
i\left(
\sum_{\begin{subarray}{c}
e \in E \\
u\in T(e) \\
\land v \in H(e)
\end{subarray}}\!\!\!\!\!\!\frac{1}{\delta_e}\,
\mathcal{F}_{u \trianglelefteq e}^{\top}\mathcal{F}_{v \trianglelefteq e}
-\!\!\!\!\!\!\sum_{\begin{subarray}{c}
e \in E \\
u \in H(e)\\
\land v \in T(e)
\end{subarray}}\!\!\!\!\!\frac{1}{\delta_e}\,
\mathcal{F}_{u \trianglelefteq e}^{\top}\mathcal{F}_{v \trianglelefteq e}
\right),
& u \neq v.
\end{array}
\right.
\end{eqnarray}

The entry $(\mathbf{L}^{\vec{\mathcal{F}}})_{uv}$ is determined by all hyperedges $e \in E$ that contain both nodes $u$ and $v$. 
From the second case in Eq.~(\ref{eq:laplacian-expanded-special}), whenever both $u$ and $v$ are both heads ($u,v \in H(e)$) or tails ($u,v \in T(e)$), the contribution to the real part $\Re\left((\mathbf{L}^{\vec{\mathcal{F}}})_{uv}\right)$ is negative and given by the opposite of the normalized weight $-\frac{1}{\delta_e}\,\mathcal{F}_{u \trianglelefteq e}^{\top}\mathcal{F}_{v \trianglelefteq e}$. 
In the undirected case, this is the only possible contribution, which matches the expected behavior of undirected hypergraph Sheaf Laplacians, where $u,v \in T(e)$ (or, equivalently, $u,v \in H(e)$). Hyperedges where $u,v$ take opposite roles contribute to the imaginary part with their weight either negatively (if $u \in T(e)$ and $v \in H(e)$) or positively (if $u \in H(e)$ and $v \in T(e)$). 
\looseness -1 Due to this, in the special case where $q = \frac{1}{4}$, $\Im\left((\mathbf{L}^{\vec{\mathcal{F}}})_{uv}\right)$ depends on the net contribution of $u$ and $v$ across all the directed hyperedges that contain them. This is in line with the ``net flow'' behavior observed in directed graphs~\citep{fiorini2023sigmanetlaplacianrule} and hypergraphs~\citep{fiorini2024let}.

When interpreted as a linear operator acting on a complex signal $\mathbf{x} \in \mathbb{C}^{nd}$, $\mathbf{L}^{\vec{\mathcal{F}}}$ reads:
\begin{equation}\label{eq:disagreement-new}
\bigl(\mathbf{L}^{\vec{\mathcal F}}(\mathbf{x})\bigr)_u 
= \sum_{e \in E:\,u \in e}
\frac{1}{\delta_e}\,
\vec{\mathcal F}_{u \trianglelefteq e}^{\dagger}
\sum_{\substack{v \in e: \\ v \neq u}}
\left(
\vec{\mathcal F}_{u \trianglelefteq e}\, \mathbf{x}_u
-
\vec{\mathcal F}_{v \trianglelefteq e}\, \mathbf{x}_v
\right).
\end{equation}
We define the \textit{Normalized Directed Sheaf Hypergraph Laplacian} $\mathbf{L}^{\vec{\mathcal{F}}}_N$ as:
\begin{equation*}
    \mathbf{L}^{\vec{\mathcal{F}}}_N := \Dvc^{-\frac12} \mathbf{L}^{\vec{\mathcal{F}}} \Dvc^{-\frac12}.
\end{equation*}
Using~\cref{eq:laplaciano}, this yields:
\begin{equation}\label{eq:normalized-laplacian}
\mathbf{L}^{\vec{\mathcal{F}}}_N =
\Dvc^{-\frac12} \underbrace{( \Dvc - \mathbf{Q}^{\vec{\mathcal{F}}})}_{\mathbf{L}^{\vec{\mathcal{F}}}}\Dvc^{-\frac12}
=
\textbf{I}_{nd} - \mathbf{Q}^{\vec{\mathcal{F}}}_N, \; \; \text{ where } \; \;
\mathbf{Q}^{\vec{\mathcal{F}}}_N  :=  \Dvc^{-\frac12} {\mathbf{B}^{(q)}}^\dagger \mathbf{D}_E^{-1} \mathbf{B}^{(q)} \Dvc^{-\frac12}.
\end{equation}

\subsection{Spectral Properties}

We now establish that our proposed \emph{Normalized Directed Sheaf Hypergraph Laplacian} $\mathbf{L}^{\vec{\mathcal{F}}}_N$ satisfies all the spectral properties required for a principled convolutional operator. Specifically, we show that $\mathbf{L}^{\vec{\mathcal{F}}}_N$ is diagonalizable, has real, nonnegative eigenvalues, is positive semidefinite, and has a bounded spectrum. These ensure that the Fourier transform is well-defined and that polynomial filters of $\mathbf{L}^{\vec{\mathcal{F}}}_N$ implement \emph{localized, stable convolutions}, in direct analogy with classical spectral-based approaches~\citep{Shuman,kipf_semi-supervised_2017,Defferrard}. Failure to preserve them can result in a complete breakdown of the connection between message passing, the Fourier transform of a (hyper-)graph and its connection with signal theory. The proofs of the claims in this and the next section are provided in \cref{app:theoretical-results}.

% Diagonalizability guarantees the existence of a complete eigenbasis, while the real-valuedness and nonnegativity of the spectrum allow $\mathbf{L}^{\vec{\mathcal{F}}}_N$ to be interpreted as a diffusion operator and ensure stability of polynomial approximations~\citep{Hammond2011,singh2023signed}.
% Finally, the boundedness of the spectrum further enables control of the convolutional operator's Lipschitz constant, which is crucial for stable training~\citep{kipf_semi-supervised_2017}.}

We begin by showing that $\mathbf{L}^{\vec{\mathcal{F}}}_N$ admits an eigenvalue decomposition with real eigenvalues:
\begin{restatable}{theorem}{diagonalizable}\label{thm:diagonalizable}
$\mathbf{L}^{\vec{\mathcal{F}}}_N$ is diagonalizable with real eigenvalues.
\end{restatable}

Next, we derive the formula of the Dirichlet energy function associated with $\mathbf{L}^{\vec{\mathcal{F}}}_N$, which provides a measure of the global smoothness of a signal $x \in \mathbb{C}^{nd}$ across the entire hypergraph:
\begin{restatable}{theorem}{dirichlet}\label{thm:dirichlet}
The Dirichlet energy induced by $\mathbf{L}^{\vec{\mathcal F}}_N$ for a signal $\mathbf{x} \in \mathbb{C}^{nd}$ is:
\[
\mathcal{E}_N(\mathbf{x}) = \mathbf{x}^{\dagger} \mathbf{L}^{\vec{\mathcal{F}}}_N \mathbf{x} = \frac12 \sum_{e\in E}\frac{1}{\delta_e}
\sum_{\substack{u,v\in e:\\ u\neq v}}
\Bigl\| \vec{\mathcal F}_{u \trianglelefteq e}\,\mathbf{D}_u^{-\frac12}\mathbf{x}_u
- \vec{\mathcal F}_{v \trianglelefteq e}\, \mathbf{D}_v^{-\frac12}\mathbf{x}_v \Bigr\|_2^2.
%\ge 0.
\]
\end{restatable}

By leveraging the two previous theorems, we show that the spectrum of the Directed Sheaf Hypergraph Laplacian $\mathbf{L}^{\vec{\mathcal{F}}}_N$ only admits (real) non-negative eigenvalues:
\begin{restatable}{corollary}{psd}\label{cor:psd}
$\mathbf{L}^{\vec{\mathcal{F}}}_N$ is positive semidefinite.
\end{restatable}

Finally, we prove that the spectrum of $\mathbf{L}^{\vec{\mathcal{F}}}_N$
is upper bounded by 1:
\begin{restatable}{theorem}{bounded}\label{thm:bounded-spectrum}
$\lambda_{\max}(\mathbf{L}^{\vec{\mathcal{F}}}_N) \leq 1$.  
\end{restatable}

\subsection{Generalization Properties}\label{sec:generalization-properties}

Beyond its spectral properties, our \textit{Directed Sheaf Hypergraph Laplacian} provides a unified definition of a Laplacian matrix that recovers and extends several existing Laplacian matrices.

First, we discuss the relationship between $\mathbf{L}^{\vec{\mathcal{F}}}$ and the \textit{Sheaf Laplacian} introduced by \cite{hansen2020sheafneuralnetworks} and highlight its connection with the classical graph Laplacian defined as $\textbf{L}:= \textbf{D} - \textbf{A}$ where $A$ is the graph's adjacency matrix and $D$ its degree matrix (see~\citep{biggs1993algebraic} for a reference) for the undirected case:
\begin{restatable}{theorem}{bodnar}\label{thm:bodnar}
For a $2$-uniform hypergraph without directions, the Laplacian operator $\mathbf{L}^{\vec{\mathcal{F}}}$ reduces to the Sheaf Laplacian \citep{hansen2020sheafneuralnetworks} (up to a scaling factor of 2) and, when considering the case of a trivial Sheaf (where $\mathcal{F}_{u \trianglelefteq e} = 1$), it coincides with the classical graph Laplacian (up to a scaling factor of 2).
\end{restatable}

With \cref{thm:magnet}, we show that $\mathbf{L}^{\vec{\mathcal{F}}}$ generalizes several Laplacians designed for directed graphs like the \textit{Magnetic Laplacian} \citep{zhang2021magnetneuralnetworkdirected} and the \textit{Sign-Magnetic Laplacian} \citep{fiorini2023sigmanetlaplacianrule}: 
\begin{restatable}{theorem}{magnet}\label{thm:magnet}
For a directed $2$-uniform hypergraph with unitary edge weights (i.e., $w_e = 1, e \in E$) with directed and undirected edges, $\mathbf{L}^{\vec{\mathcal{F}}}$ recovers, as a special case, the Magnetic Laplacian \citep{zhang2021magnetneuralnetworkdirected} for any $q \in \mathbb{R}$ and the Sign-Magnetic Laplacian \citep{fiorini2023sigmanetlaplacianrule} when $q = \frac{1}{4}$.
\end{restatable}

In the context of hypergraphs, our operator naturally recovers existing hypergraph Laplacians. We begin by showing that it recovers the undirected hypergraph Laplacian of \citet{zhou}:
\begin{restatable}{theorem}{zhou}\label{thm:zhou}
Given a hypergraph $\mathcal{H}$ (directed or undirected), the normalized Directed Hypergraph Laplacian $\mathbf{L}^{\vec{\mathcal{F}}}_N$ recovers, as a special case, the undirected hypergraph Laplacian of \cite{zhou}.
\end{restatable}
We show that $\mathbf{L}^{\vec{\mathcal{F}}}$ generalizes the \textit{Generalized Directed Laplacian} proposed by \cite{fiorini2024let}:
\begin{restatable}{theorem}{gedi}\label{thm:gedi}
Given a directed hypergraph $\mathcal{H}$ with unitary weights associated to each hyperedge (i.e., $w_e = 1$), the Normalized Directed Sheaf Hypergraph Laplacian $\mathbf{L}^{\vec{\mathcal{F}}}_N$ recovers, as a special case, the Generalized Directed Laplacian $\vec{\mathbf{L}}_N$ of~\cite{fiorini2024let}.
\end{restatable}

%PREVIOUS
% Crucially, our Laplacian operator does not recover the (linear) Sheaf Hypergraph Laplacian proposed by~\citet{duta}. This is intentional, since \citet{duta}'s formulation fails to satisfy the spectral properties that are required to build a well-defined convolutional operator, most notably positive semidefiniteness (which only holds in the the 2-uniform case). This is better described in~\cref{app:iulia-sheaf}.
% %
% Consequently, when either setting \(q=0\) or considering undirected hypergraphs, our operator enjoys all the properties that are required of a principled Laplacian matrix and, thus, it provides (to the best of our knowledge) the first definition of a Sheaf Hypergraph Laplacian suitable for undirected hypergraphs.

Crucially, our Laplacian operator does not recover the (linear) Sheaf Hypergraph Laplacian proposed in~\citep{duta}. This difference is intentional: the formulation proposed in their work, defined for any pair of nodes $u,v \in V$ as:
\begin{equation}
(\mathbf{L}^{\mathcal{F}})_{uu}
= \sum_{e:\,u\in e}\frac{1}{\delta_e}\,
\mathcal{F}_{u \trianglelefteq e}^{\top}\mathcal{F}_{u \trianglelefteq e},
\qquad
(\mathbf{L}^{\mathcal{F}})_{uv}
= -\!\sum_{\begin{subarray}{c} e: u,v\in e \\ v \neq u\end{subarray}}
\frac{1}{\delta_e}\,
\mathcal{F}_{u \trianglelefteq e}^{\top}\mathcal{F}_{v \trianglelefteq e},
\end{equation}
fails to satisfy several key spectral properties required to construct a well-defined convolutional operator, most notably positive semidefiniteness. Comparing this definition with \cref{eq:laplacian-expanded}, it is evident that, for undirected hypergraphs or when $q = 0$, the two operators differ in the diagonal term: our formulation introduces the factor $(1 - \tfrac{1}{\delta_e})$ rather than $\tfrac{1}{\delta_e}$. Their definition suffices to guarantee such properties only in the special case of 2-uniform hypergraphs (i.e, standard graphs), however this discrepancy can produce negative eigenvalues for general cases, preventing the matrix from being interpreted as a diffusion operator and precluding the definition of a stable Fourier transform.
Polynomial filters based on the Laplacian of~\citet{duta} therefore cannot guarantee localized and stable convolutions, limiting its general applicability. A more detailed analysis, including an example that demonstrates this drawback, is presented in \cref{app:iulia-sheaf}.

In contrast, when setting \(q=0\) or considering undirected hypergraphs, our operator satisfies all the properties required of a principled Laplacian matrix and, to the best of our knowledge, provides the first definition of a Sheaf Hypergraph Laplacian suitable for undirected hypergraphs.
\section{Directional Sheaf Hypergraph Network} \label{sec:directional-sheaf-hypergraph-networks}

In this section, we describe our proposed {\em Directional Sheaf Hypergraph Network} (\approach).
Inspired by~\cite{hansen2020sheafneuralnetworks}, we define the \emph{sheaf diffusion process} on a directed hypergraph $\mathcal{H}$ as an extension of classical heat diffusion, which plays a central role in spectral-based GNN convolution operators~\citep{kipf_semi-supervised_2017}.
Starting from the differential equation:
$$
\dot{\mathbf{X}}_t = -\,\mathbf{L}^{\vec{\mathcal{F}}}_N \mathbf{X_t}.
$$ 
and applying a unit-step Euler discretization, we obtain:
$$
\mathbf{X_{t+1}} = \mathbf{X_t} - \mathbf{L}^{\vec{\mathcal{F}}}_N \mathbf{X_t} = (\mathbf{I}_{nd} - \mathbf{L}^{\vec{\mathcal{F}}}_N) \mathbf{X_t}.
$$
By introducing learnable parameters and a nonlinear activation function $\sigma$, this discrete diffusion process leads to the following equation of the convolutional layer of \approach:
\begin{equation}\label{eq:diffusion}
\mathbf{X}_{t+1} 
= \sigma\left( (\mathbf{I}_{nd} - \mathbf{L}^{\vec{\mathcal{F}}}_N)\,(\mathbf{I}_n \otimes \mathbf{W_1})\, \mathbf{X}_t\, \mathbf{W}_2 \right) 
= \sigma\left( \mathbf{Q}^{\vec{\mathcal{F}}}_N\,(\mathbf{I}_n \otimes \mathbf{W}_1)\, \mathbf{X}_t\, \mathbf{W}_2 \right)  \in \mathbb{C}^{nd \times f},
\end{equation}
where $\mathbf{W_1} \in \mathbb{R}^{d \times d}$ and $\mathbf{W_2} \in \mathbb{R}^{f \times f}$ are trainable weight matrices, and $\otimes$ denotes the Kronecker product. $\mathbf{X}_0$ is obtained from the input matrix of node features of size $n \times f$, to which a linear projection is applied to produce an $n \times (df)$ matrix, which is then reshaped into a $(nd) \times f$ matrix before applying the diffusion process.
Finally, since the convolutional layer operates in the complex domain, we transform the output of the final convolutional layer into a real-valued representation right before feeding it to the classification head. Following \citet{zhang2021magnetneuralnetworkdirected,fiorini2023sigmanetlaplacianrule}, we do so by applying an \emph{unwind} operation, concatenating the real and imaginary components of the complex features as follows:
\begin{equation*}\label{eq:unwind}
\mathrm{unwind}(\mathbf{X}) = \Re(\mathbf{X})  \,\|\, \Im(\mathbf{X}) \in \mathbb{R}^{n \times 2f}
\end{equation*}

\paragraph{Restriction Maps}

The expressive power of sheaf-theoretical approaches lies in their ability to define a diffusion operator via a learnable $d \times d$ restriction map associated to each node–edge pair.
Following \cite{bodnar2023neuralsheafdiffusiontopological}, we learn each restriction map as a function of the corresponding node and hyperedge features. In particular, for each hyperedge $e \in E$ and each node $u \in e$, each directionless restriction map $\mathcal{F}_{v \trianglelefteq e}$ is parametrized as $\mathcal{F}_{v \trianglelefteq e} = \Phi(\mathbf{x}_v  \,\|\, \mathbf{x}_e) \in \mathbb{R}^{d \times d}$, where $x_u$ is the node feature and $\mathbf{x}_e$ is the hyperedge feature and $\Phi$ is an MLP. If hyperedge features are not explicitly provided, they are computed by aggregating the features of the hyperedge's nodes via a mean or a sum.
Given that the node and the hyperedge features $\mathbf{x}_v$ and $\mathbf{x}_e$ are complex-valued due to \cref{eq:diffusion}, we employ the same \emph{unwind} operation to map them into a form suitable for input to $\Phi$. Activation functions can be either \emph{sigmoid} or \emph{tanh}.

\paragraph{\light}
A key objective in our model design is to balance predictive performance with computational efficiency. Constructing an $nd \times nd$ Laplacian matrix from $d \times d$ restriction maps inherently increases complexity, an issue already noted in the graph setting~\citep{bodnar2023neuralsheafdiffusiontopological}.
To mitigate this, we introduce \light, a variant of \approach that achieves competitive, and in some cases superior, results across several datasets (\cref{tab:results_real_world,tab:results_synthetic}) at a significantly lower computational cost (see \cref{app:computational-complexity}). In it, we detach the gradient computation during the Laplacian's construction: this way, the model continues to rely on the predicted restriction maps, but avoids costly gradient propagation. In \light, the parameters of the MLP responsible for predicting the restriction maps (which are encoded in $\Phi(\cdot)$) remain fixed throughout the training process. The model's adaptability arises from the initial projection layer, which embeds the inputs into a shared feature space where they can be more effectively processed through these parameters (see \cref{app:implementation-details} for a better visualization).
This phenomenon aligns with insights from the literature on overparameterization \citep{arora2019fine} and extreme learning machines \citep{huang2006extreme}, where fixed random projections can still yield strong generalization due to the expressive power of the input embeddings.
Further details on the difference between the two approaches are discussed in \cref{app:implementation-details}.
\paragraph{Computational Complexity}
We provide an estimate of the asymptotic complexity of our model at inference time. Let $n$ denote the number of nodes, $m$ the number of hyperedges, $d$ the stalk dimension, $c$ the product of input and output feature dimensions in the linear transformation, $\bar e$ the average hyperedge size, and $\bar v$ the average number of hyperedges a node participates in. Summing the contributions from the feature transformation, message passing, learning of restriction maps, and Laplacian assembly, the overall complexity is $\mathcal{O}(n(c^{2}+d)+m(\bar e d+\bar e^{2}(d+c)+\bar v c))$ for diagonal maps, and $\mathcal{O}(n(c^{2}+d^{3})+m(\bar e d^{3}+\bar e^{2}(d^{3}+dc)+\bar v d^{2} c))$ for non-diagonal maps, sharing the same asymptotic complexity as SheafHyperGNN by \cite{duta}. For a comprehensive analysis of the contributions leading to this asymptotic complexity, we refer the reader to \cref{app:computational-complexity}.

\section{Experimental Evaluation}
We evaluate \approach and \light against \nbaselines baseline models from both the directed and undirected hypergraph literature on real-world datasets (\cref{subsec:real-world datasets}) as well as synthetic datasets (\cref{subsec:synthetic datasets}) for the node classification task. From the \emph{undirected} hypergraph-learning literature, we include HGNN~\citep{feng2019hypergraph}, HNHN~\citep{HNHN2020}, UniGCNII~\citep{ijcai21-UniGNN}, LEGCN~\citep{LEGCN}, HyperND~\citep{tudisco}, AllDeepSets and AllSetTransformer~\citep{allset}, ED-HNN~\citep{wang2022equivariant}, SheafHyperGNN~\citep{duta}, and PhenomNN~\citep{wang2023hypergraphenergyfunctionshypergraph}. From the \emph{directed} hypergraph-learning literature, we consider GeDi-HNN~\citep{fiorini2024let} and DHGNN~\citep{DHGNN}, along with a variant, as baselines.
Model performance is measured in terms of classification accuracy. Following the standard practice in the literature \citep{allset,wang2022equivariant,fiorini2024let}, we adopt a 50\%/25\%/25\% split for training, validation, and testing, respectively, and, for each model, we report the average test accuracy and the standard deviation over 10 independent runs. Details on the baselines, hyperparameter tuning, and the experimental setup are provided in \cref{app:exp_setup}.

\subsection{Real-World Datasets} \label{subsec:real-world datasets}

To evaluate our models on real-world datasets, we follow the pre-processing procedure introduced by~\cite{tran2022directedhypergraphneuralnetwork} and \cite{fiorini2024let}, and apply it to a suite of publicly available directed graph benchmarks to obtain their directed hypergraph counterparts for performing the node classification task (see \cref{app:transformation}). The considered datasets are:
\texttt{Cora}~\citep{zhang2022hypergraphconvolutionalnetworksequivalency}, \texttt{email-Enron}, \texttt{email-EU}~\citep{benson-simplicial}, \texttt{Telegram}~\citep{bovet-telegram}, \texttt{Chameleon}, \texttt{Squirrel}, and \texttt{Roman-empire}.
Due to space limitations, \cref{tab:results_real_world} includes only the datasets that yield the most interesting insights. Additional results can be found in \cref{tab:results_real_world_complete}, while additional informations on the datasets are provided in \cref{app:datasets-description}.

% These datasets range in CE homophily from $h_{CE} = 0.23$ on the \texttt{Roman} dataset to $h_{CE} = 0.80$ on the \texttt{Cora} dataset (see \cref{app:datasets-description}). 
%

\begin{table*}[ht]
\centering
\caption{Mean accuracy $\pm$ standard deviation on node classification datasets. For each dataset, the best result is shown in \textbf{bold}, and the second best is \second{underlined}.}
\label{tab:results_real_world}
\resizebox{\columnwidth}{!}{
\begin{tabular}{lccccccc}
\toprule & 
\textbf{Roman-empire} & \textbf{Squirrel} & \textbf{email-EU} & \textbf{Telegram} & \textbf{Chameleon} & \textbf{email-Enron} & \textbf{Cora} \\
% \textbf{CE Homophily}   & 0.2363 & 0.2448 & 0.2608 & 0.2854 & 0.3221 & 0.3251 & 0.8035 \\
% \textbf{\# Nodes}       & 22{,}662 & 2{,}223 & 986 & 245 & 890 & 143 & 2{,}708 \\
% \textbf{\# Hyperedges}  & 22{,}662 & 2{,}060 & 787 & 183 & 797 & 139 & 1{,}565 \\
\midrule
HGNN              & $38.44 \pm 0.44$ & $35.47 \pm 1.44$ & $48.91 \pm 3.11$ & $51.73 \pm 3.38$ & $39.98 \pm 2.28$ & $52.85 \pm 7.27$ & $87.25 \pm 1.01$ \\
HNHN              & $46.07 \pm 1.22$ & $35.62 \pm 1.30$ & $29.68 \pm 1.68$ & $38.22 \pm 6.95$ & $35.81 \pm 3.23$ & $18.64 \pm 6.90$ & $78.16 \pm 0.98$ \\
UniGCNII          & $78.89 \pm 0.51$ & $38.28 \pm 2.56$ & $44.98 \pm 2.69$ & $51.73 \pm 5.05$ & $39.85 \pm 3.19$ & $47.43 \pm 7.47$ & $87.53 \pm 1.06$ \\
LEGCN             & $65.60 \pm 0.41$ & $39.18 \pm 1.54$ & $32.91 \pm 1.83$ & $45.38 \pm 4.23$ & $39.29 \pm 2.04$ & $37.03 \pm 7.16$ & $74.96 \pm 0.94$ \\
HyperND           & $68.31 \pm 0.69$ & $40.13 \pm 1.85$ & $32.79 \pm 2.90$ & $44.62 \pm 5.49$ & $44.95 \pm 3.20$ & $38.11 \pm 7.69$ & $78.48 \pm 1.02$ \\
AllDeepSets       & $81.79 \pm 0.72$ & $40.69 \pm 1.90$ & $37.37 \pm 6.29$ & $49.19 \pm 6.73$ & $42.97 \pm 3.60$ & $37.29 \pm 7.90$ & $86.86 \pm 0.85$ \\
AllSetTransformer & $83.53 \pm 0.64$ & $40.53 \pm 1.33$ & $38.26 \pm 3.57$ & $66.92 \pm 4.36$ & $43.85 \pm 5.42$ & $63.78 \pm 3.66$ & $86.73 \pm 1.13$ \\
ED\mbox{-}HNN     & $83.82 \pm 0.31$ & $39.85 \pm 1.79$ & $68.91 \pm 4.00$ & $60.38 \pm 3.86$ & $44.67 \pm 2.33$ & $51.35 \pm 6.04$ & $86.94 \pm 1.25$ \\
SheafHyperGNN     & $74.50 \pm 0.57$ & $42.01 \pm 1.11$ & $52.78 \pm 9.13$ & $70.00 \pm 5.32$ & $41.06 \pm 4.94$ & $63.51 \pm 5.95$ & $87.15 \pm 0.64$ \\
PhenomNN          & $71.22 \pm 0.45$ & $39.45 \pm 2.19$ & $37.69 \pm 4.40$ & $47.69 \pm 6.59$ & $43.62 \pm 4.29$ & $47.02 \pm 6.75$ & \first{$88.12 \pm 0.86$} \\
\midrule
GeDi\mbox{-}HNN   & \second{$83.87 \pm 0.63$} & $43.02 \pm 3.00$ & $52.31 \pm 2.84$ & $77.12 \pm 4.82$ & $39.29 \pm 2.04$ & $50.54 \pm 5.80$ & $85.16 \pm 0.94$ \\
DHGNN             & $77.58 \pm 0.54$ & $39.85 \pm 1.79$ & $32.35 \pm 2.93$ & $79.62 \pm 5.78$ & $44.08 \pm 4.11$ & $42.16 \pm 8.04$ & $83.16 \pm 1.33$ \\
DHGNN (w/ emb.)   & $22.50 \pm 0.81$ & $40.33 \pm 1.42$ & $55.10 \pm 3.48$ & $80.58 \pm 3.89$ & $40.85 \pm 2.76$ & $58.38 \pm 7.57$ & $73.12 \pm 1.04$ \\
\midrule
\textbf{\approach}         & OOM & \second{$43.55 \pm 2.87$} & \second{$78.62 \pm 2.50$} & \first{$88.65 \pm 5.54$} & \first{$47.02 \pm 4.35$} & \second{$75.68 \pm 3.42$} & $87.84 \pm 0.90$ \\
\textbf{\light}            & \first{$89.24 \pm 0.57$} & \first{$44.09 \pm 2.36$} & \first{$82.67 \pm 1.29$} & \second{$81.15 \pm 4.19$} & \second{$46.50 \pm 4.09$} & \first{$76.76 \pm 2.48$} & \second{$88.02 \pm 1.11$} \\
\bottomrule
\end{tabular}
}
\end{table*}

\approach, and its variant \light, which both leverage the theoretical advantages of associating a Directed Cellular Sheaf to a directed hypergraph, consistently outperform the \nbaselines baselines from both the undirected and directed hypergraph learning literature on \ndatasetsgood out of \ndatasetstotal real-world datasets. The largest relative gains are observed on the \texttt{email-Enron} and \texttt{email-EU} datasets, where \approach and \light improve over the best baseline by up to 20\%. A substantial improvement is also achieved on \texttt{Telegram}, confirming the importance of directional information in this benchmark where all directed methods perform strongly. More moderate but consistent improvements are found on highly heterophilic datasets such as \texttt{Roman-empire}, \texttt{Chameleon} and \texttt{Squirrel} while on highly homophilic datasets such as \texttt{Cora} performance is on par with the strongest baselines. As shown in Table \ref{tab:q_values_table}, the charge parameter $q$ selected by the hyperparameter-selection procedure for our models on highly homophilic datasets is consistently $0.0$. This observation is in line with the findings of \cite{zhang2021magnetneuralnetworkdirected}, who report that, in such settings, directional information behaves as noise for node classification.

A better visualization of the impact of the charge parameter $q$ on the predictive performance of our model can be found in Fig. \ref{fig:influence_q_both}, where we highlight the positive impact of directional information on the \texttt{Telegram} dataset and how direction is detrimental on the \texttt{Cora} dataset. These results not only demonstrate the effectiveness of our models in highly heterophilic settings, but also show how integrating the concept of directionality in hypergraphs can substantially improve performance. Moreover, unlike GeDi-HNN and DHGNN, which are based on Laplacian formulations that embed directionality without any degree of freedom, in our models one can flexibly choose the relevance of directional information by a suitable choice of the charge parameter $q$.
\subsection{Synthetic Datasets} \label{subsec:synthetic datasets}
We additionally evaluate our models on the synthetic datasets introduced by~\cite{fiorini2024let}, built over \(n=500\) nodes and split into \(c=5\) classes. Each class contains 30 random intra-class hyperedges, while inter-class directed hyperedges, consisting of multiple tail and head nodes, are added between class pairs with sizes drawn uniformly from \(\{3,\ldots,10\}\). By varying the number of inter-class hyperedges \(I_o \in \{10,30,50\}\), we control the strength of directional connectivity. This design provides a clean benchmark to test the models’ ability to capture directionality; further details are given in \cref{app:datasets-description}.
\begin{figure*}[t] 
\begin{minipage}[t]{0.50\textwidth}
\captionsetup{type=table}
\caption{Mean accuracy $\pm$ standard deviation \\on the synthetic datasets.}
\small
\scriptsize
\setlength{\tabcolsep}{0.8pt}
\renewcommand{\arraystretch}{1.1}
\label{tab:results_synthetic}
\begin{tabular}{lccc}
\toprule
\textbf{Method} & \textbf{$I_o=10$} & \textbf{$I_o=30$} & \textbf{$I_o=50$} \\
\midrule
HGNN & $47.12 \pm 5.37$ & $43.44 \pm 6.63$ & $37.76 \pm 7.72$ \\
HNHN & $20.40 \pm 2.93$ & $28.88 \pm 9.45$ & $19.76 \pm 3.85$ \\
UniGCNII & $21.44 \pm 4.33$ & $21.12 \pm 2.95$ & $19.84 \pm 2.34$ \\
LEGCN & $17.60 \pm 2.43$ & $20.72 \pm 3.48$ & $19.60 \pm 2.82$ \\
HyperND & $20.40 \pm 2.93$ & $21.12 \pm 3.20$ & $20.64 \pm 1.92$ \\
AllDeepSets & $44.40 \pm 6.81$ & $32.32 \pm 4.82$ & $31.70 \pm 5.92$ \\
AllSetTransformer & $21.12 \pm 3.79$ & $43.68 \pm 8.72$ & $31.84 \pm 3.31$ \\
ED-HNN & $34.00 \pm 6.05$ & $18.88 \pm 2.56$ & $32.48 \pm 6.17$ \\
SheafHyperGNN & $30.64 \pm 5.39$ & $27.28 \pm 7.31$ & $26.00 \pm 9.59$ \\
PhenomNN & $22.24 \pm 4.73$ & $22.08 \pm 4.20$ & $18.72 \pm 3.22$ \\
\midrule
GeDi\mbox{-}HNN & $71.44 \pm 3.14$ & $71.84 \pm 3.31$ & $78.24 \pm 5.64$ \\
DHGNN & $40.72 \pm 4.55$ & $51.68 \pm 3.97$ & $35.76 \pm 3.70$ \\
DHGNN (w/ emb.) & $84.48 \pm 3.22$ & $85.28 \pm 3.32$ & $81.12 \pm 3.22$ \\
\midrule
\textbf{\approach} & \second{$94.96 \pm 1.75$} & \first{$97.84 \pm 1.86$} & \second{$95.84 \pm 2.17$} \\
% Optimal $q$ & 0.25 & 0.1 & 0.1 \
% \midrule
\textbf{\light} & \first{$95.60 \pm 2.15$} & \second{$ 97.04 \pm 2.79$} & \first{$99.04 \pm 0.86$} \\

\bottomrule
\end{tabular}
\end{minipage}
\begin{minipage}[t]{0.45\textwidth}
\vspace{0pt}
\captionsetup{type=figure}
\caption{Effect of the charge parameter $q$ on Telegram and Cora.} 
\label{fig:influence_q_both}
\centering

\definecolor{mplblue}{HTML}{1F77B4}

\resizebox{.99\linewidth}{!}{
\begin{tikzpicture}
\begin{axis}[
    title={Telegram dataset},
    title style={font=\small, color=black},
    label style={font=\small},
    tick label style={font=\small},
    width=\linewidth,
    height=.36\linewidth,   
    scale only axis,
    clip=true,
    xlabel={Charge $q$},
    ylabel={Accuracy (\%)},
    xmin=0, xmax=0.25,
    ymin=63.46, ymax=88.90,
    ytick={60,65,70,75,80,85},
    xtick={0.00,0.05,0.10,0.15,0.20,0.25},
    xticklabel style={/pgf/number format/fixed,/pgf/number format/precision=2},
    xmajorgrids=true, ymajorgrids=true,
    grid style={line width=.1pt, draw=gray!20},
    major grid style={line width=.2pt, draw=gray!40},
    axis lines=left,
]
\addplot[draw=blue, mark=*, line width=1pt] coordinates {
    (0,63.46) (0.05,67.31) (0.10,76.92)
    (0.15,80.38) (0.20,82.62) (0.25,88.50)
};
\end{axis}
\end{tikzpicture}%
}

% --- Cora
\resizebox{.99\linewidth}{!}{%
\begin{tikzpicture}
\begin{axis}[
    title={Cora dataset},
    title style={font=\small, color=black},
    label style={font=\small},
    tick label style={font=\small},
    width=\linewidth,
    height=.36\linewidth,
    scale only axis,
    clip=true,
    xlabel={Charge $q$},
    ylabel={Accuracy (\%)},
    xmin=0, xmax=0.25,
    ymin=83.37, ymax=86.8,
    ytick={84,85,86},
    xtick={0.00,0.05,0.10,0.15,0.20,0.25},
    xticklabel style={/pgf/number format/fixed,/pgf/number format/precision=2},
    xmajorgrids=true, ymajorgrids=true,
    grid style={line width=.1pt, draw=gray!20},
    major grid style={line width=.2pt, draw=gray!40},
    axis lines=left,
]
\addplot[draw=blue, mark=*, line width=1pt] coordinates {
    (0,86.25) (0.05,86.06) (0.10,85.96)
    (0.15,85.63) (0.20,84.74) (0.25,83.37)
};
\end{axis}
\end{tikzpicture}%
}
\end{minipage}
\end{figure*}

The results in Table~\ref{tab:results_synthetic} clearly demonstrate the advantage of our models \approach and \light over existing baselines. Classical undirected hypergraph methods are unable to capture the directional structure that dominate these benchmarks, and as a result their performance is limited. Directed methods such as GeDi-HNN and DHGNN achieve stronger results, confirming the importance of explicitly incorporating directionality into the convolutional process. Yet, \approach and \light, which provide a principled and more expressive treatment of directional structure, yield consistent improvements across all synthetic datasets, outperforming the strongest directed baselines by up to 18 percentage points and reaching 
99.04\% accuracy on the third synthetic dataset---this highlights the expressive power that the notion of Directed Hypergraph Cellular Sheaves unlocks. 
%
% Moreover, a clear trend emerges: as the number of directed hyperedges increases, the performance of both GeDi-HNN and our approach improves, indicating that richer directional structure provides more exploitable signal for directed hypergraph learning.
%
\section{Conclusion and Future Works}
We introduced the concept of \emph{Directed Hypergraph Cellular Sheaves} for directed hypergraphs and derived the corresponding \emph{Directed Sheaf Hypergraph Laplacian}, which we integrated into our proposed framework \approach. By encoding hyperedge direction via a topology-aware complex-valued inductive bias, our method naturally accommodates both directed and undirected hypergraphs while also unifying and generalizing several operators from the graph and hypergraph learning literature.
Across a broad set of benchmark datasets, \approach consistently outperforms methods from both the directed and undirected hypergraph learning literature.
As future work, a natural step forward is to evaluate our framework on larger and \emph{natively directed hypergraph datasets} such as protein-protein interaction networks to further test the scalability and expressivity of the method, possibly employing Language Models (LMs) to generate features.
Finally, an intriguing direction is to make the \emph{charge parameter} $q$ directly learnable, allowing each layer to adapt its diffusion process dynamically.
\section*{Acknowledgement of Support}
Antonio Purificato, Federico Siciliano and Fabrizio Silvestri acknowledge projects FAIR (PE0000013), under the MUR National Recovery and Resilience Plan funded by the European Union - NextGenerationEU, and project NEREO (Neural Reasoning over Open Data), funded by the Italian Ministry of Education and Research (PRIN) Grant no. 2022AEFHAZ.
Stefano Coniglio's work was partially supported by the European Union under Next Generation EU — the Italian National Recovery and Resilience Plan (PNRR), PRIN 2022 PNRR (project code P20227CTY3, CUP D53D23018800001), project title "HEXAGON: Highly-specialized EXact Algorithms for Grid Operations at the National level".

\section*{Reproducibility Statement}
We provide all the necessary information to facilitate the reproducibility of our results.
%We did all our best in order to make this paper fully reproducible.
Our code repository code can be found \href{https://github.com/EmaMule/DirectionalSheafHypergraphs}{here}. The README contains all that is needed to set up the Python environment and run the experiments with the different configurations. Further details on the Experimental Setup can be found in \cref{app:exp_setup}.

\section*{Ethics Statement}
All datasets employed in this work are publicly available for research and contain no personally identifiable
information or harmful content (see \cref{app:datasets-description} for further details). The methods introduced in this paper have a societal impact comparable to that of other graph neural networks.

\section*{LLM Usage Statement}
All technical content presented in this paper is entirely our own work, with LLMs serving only as an editorial tool. No scientific content or research findings were generated using an LLM.

\bibliography{biblio}
\bibliographystyle{iclr2026_conference}

\clearpage
\appendix

\section{Theoretical Results}\label{app:theoretical-results}

\subsection{Spectral Properties}

The following Lemma (which we state for clarity even if it is not reported in the paper as a lemma) derives the expression of our proposed Laplacian matrix $\mathbf{L}^{\vec{\mathcal{F}}}$ when applied as a linear operator on a signal:
\begin{lemma} \label{app:disagreement}
Let $x \in \mathbb{C}^{nd}$ be a complex-valued signal. Component-wise, the application of $\mathbf{L}^{\vec{\mathcal{F}}}$ to it and to its normalized counterpart reads:
\begin{equation*}
\bigl(\mathbf{L}^{\vec{\mathcal{F}}}(\mathbf{x})\bigr)_u =
\sum_{e:\,u \in e}
\frac{1}{\delta_e}\,
\vec{\mathcal F}_{u \trianglelefteq e}^{\dagger}
\sum_{\substack{v \in e \\ v \neq u}}
\left(
\vec{\mathcal F}_{u \trianglelefteq e}\, \mathbf{x}_u
-
\vec{\mathcal F}_{v \trianglelefteq e}\, \mathbf{x}_v
\right).
\end{equation*}
\[
\bigl(\mathbf{L}^{\vec{\mathcal F}}_N(\mathbf{x})\bigr)_u
=
\sum_{e:\,u\in e}\frac{1}{\delta_e}\,
\Bigl(\mathbf{D}_u^{-\frac12}\vec{\mathcal F}_{u \trianglelefteq e}^{\dagger}\Bigr)
\sum_{\substack{v\in e \\ v\neq u}}
\left(
\vec{\mathcal F}_{u \trianglelefteq e}\,\mathbf{D}_u^{-\frac12}\mathbf{x}_u
-
\vec{\mathcal F}_{v \trianglelefteq e}\,\mathbf{D}_u^{-\frac12}\mathbf{x}_v
\right).
\]
\end{lemma}
\begin{proof}
We start by applying the definition of the $\mathbf{L}^{\vec{\mathcal{F}}}$ component-wise as in \cref{eq:laplacian-expanded}:
\begin{eqnarray*}
\bigl(\mathbf{L}^{\vec{\mathcal F}}(\mathbf{x})\bigr)_u 
&
% laplaciano su tutti i vertici V
=& \sum_{v \in V} (\mathbf{L}^{\vec{\mathcal F}})_{uv} \, \mathbf{x}_v \\ 
&
% applichiamo la definizione del laplaciano component wise 
= & \sum_{e:\,u \in e} 
\left(1 - \frac{1}{\delta_e}\right)
\vec{\mathcal{F}}_{u \trianglelefteq e}^{\dagger} \vec{\mathcal{F}}_{u \trianglelefteq e} \, \mathbf{x}_u
-
\sum_{e:\,u \in e} \sum_{\substack{v \in e \\ v \neq u}} 
\frac{1}{\delta_e}\,
\vec{\mathcal F}_{u \trianglelefteq e}^{\dagger} 
\vec{\mathcal F}_{v \trianglelefteq e}\, \mathbf{x}_v \\ 
&
% tiriamo fuori 1/delta_e e la sommatoria per u, essendo che compare in entrambi i termini
= & \sum_{e:\,u \in e}
\frac{1}{\delta_e}\,
\left(
(\delta_e-1)\,
\vec{\mathcal F}_{u \trianglelefteq e}^{\dagger} 
\vec{\mathcal F}_{u \trianglelefteq e}\, \mathbf{x}_u
-
\sum_{\substack{v \in e \\ v \neq u}} 
\vec{\mathcal F}_{u \trianglelefteq e}^{\dagger} 
\vec{\mathcal F}_{v \trianglelefteq e}\, \mathbf{x}_v
\right) \\ 
&
% tiriamo fuori Fu dato che non dipende da nulla nei termini interni, fuori dalla sommatoria
=& \sum_{e:\,u \in e}
\frac{1}{\delta_e}\,
\vec{\mathcal F}_{u \trianglelefteq e}^{\dagger}
\left(
(\delta_e-1)\,
\vec{\mathcal F}_{u \trianglelefteq e}\, \mathbf{x}_u
-
\sum_{\substack{v \in e \\ v \neq u}} 
\vec{\mathcal F}_{v \trianglelefteq e}\, \mathbf{x}_v
\right).
\end{eqnarray*}
% passaggio più strano: dato che abbiamo delta_e - 1 possiamo esprimere LA STESSA QUANTITA come sommatoria di v in e con v /= u (non c'è dipendenza da v)
%
Finally, notice that the coefficient $\delta_e - 1$ is exactly the number of vertices in $e$ different from $u$. Thus, the term $(\delta_e-1)\, \vec{\mathcal F}_{u \trianglelefteq e}\, \mathbf{x}_u$ can be written as a sum of $\vec{\mathcal F}_{u \trianglelefteq e}\, \mathbf{x}_u$ over all $v \in e, v \neq u$. Substituting this back, we obtain:
\begin{equation*}
\sum_{e:\,u \in e}
\frac{1}{\delta_e}\,
\vec{\mathcal F}_{u \trianglelefteq e}^{\dagger}
\left(
\sum_{\substack{v \in e \\ v \neq u}} 
\vec{\mathcal F}_{u \trianglelefteq e}\, \mathbf{x}_u
-
\sum_{\substack{v \in e \\ v \neq u}} 
\vec{\mathcal F}_{v \trianglelefteq e}\, \mathbf{x}_v
\right)
\end{equation*}
%entrambe le sommatorie interne sono uguali, possiamo raccoglierla come delta di un singolo contributo
\begin{equation*}
= \sum_{e:\,u \in e}
\frac{1}{\delta_e}\,
\vec{\mathcal F}_{u \trianglelefteq e}^{\dagger}
\sum_{\substack{v \in e \\ v \neq u}}
\left(
\vec{\mathcal F}_{u \trianglelefteq e}\, \mathbf{x}_u
-
\vec{\mathcal F}_{v \trianglelefteq e}\, \mathbf{x}_v
\right).
\end{equation*}
The linear expression for the normalized case can be derived analogously.
\end{proof}

In the remainder of this section, we report a proof for each of the theorems we stated in the paper.

\diagonalizable*
\begin{proof}
The claim follows rather directly since, as it is not hard to see, $\mathbf{L}^{\vec{\mathcal F}}_N$ is Hermitian  by construction.
\end{proof}

\dirichlet*
\begin{proof}
By definition of the Dirichlet energy as the quadratic form associated with $\mathbf{L}^{\vec{\mathcal F}}_N$, we have:
\[
\mathcal{E}_N(\mathbf{x}) = \mathbf{x}^\dagger \mathbf{L}^{\vec{\mathcal F}}_N \mathbf{x}
= \sum_{u\in V}\mathbf{x}_u^\dagger \bigl(\mathbf{L}^{\vec{\mathcal F}}_N(\mathbf{x})\bigr)_u.
\]
By substituting for $(\mathbf{L}^{\vec{\mathcal F}}_N (\mathbf{x}))_u$ (see the previous lemma), we have:
\begin{align*}
\mathcal{E}_N(\mathbf{\mathbf{x}})
&= \sum_{u\in V} \sum_{e:\, u\in e}\frac{1}{\delta_e}
   \sum_{\substack{v\in e \\ v\neq u}}
   \bigl(\vec{\mathcal F}_{u \trianglelefteq e}\mathbf{D}_u^{-\tfrac12}\mathbf{x}_u\bigr)^\dagger
   \left(
    \vec{\mathcal F}_{u \trianglelefteq e}\,\mathbf{D}_u^{-\tfrac12}\mathbf{x}_u
    -
    \vec{\mathcal F}_{v \trianglelefteq e}\,\mathbf{D}_v^{-\tfrac12}\mathbf{x}_v
    \right).
\end{align*}
Distributing the product, we obtain:
\begin{align*}
\mathcal{E}_N(\mathbf{x})
&= \sum_{e\in E}\frac{1}{\delta_e}
   \sum_{u\in e}\sum_{\substack{v\in e \\ v\neq u}}
   \bigl(\vec{\mathcal F}_{u \trianglelefteq e}\mathbf{D}_u^{-\tfrac12}\mathbf{x}_u\bigr)^\dagger
   \vec{\mathcal F}_{u \trianglelefteq e}\mathbf{D}_u^{-\tfrac12}\mathbf{x}_u
   -
   \sum_{e\in E}\frac{1}{\delta_e}
   \sum_{\substack{u,v\in e \\ u\neq v}}
   \bigl(\vec{\mathcal F}_{u \trianglelefteq e}\mathbf{D}_u^{-\tfrac12}\mathbf{x}_u\bigr)^\dagger
   \vec{\mathcal F}_{v \trianglelefteq e}\mathbf{D}_v^{-\tfrac12}\mathbf{x}_v \\[4pt]
&= \sum_{e\in E}\frac{1}{\delta_e}
   \left(
   \sum_{u\in e}\sum_{\substack{v\in e \\ v\neq u}}
   \bigl\|\vec{\mathcal F}_{u \trianglelefteq e}\mathbf{D}_u^{-\tfrac12}\mathbf{x}_u\bigr\|_2^2
   - \sum_{\substack{u,v\in e \\ u\neq v}}
   \bigl(\vec{\mathcal F}_{u \trianglelefteq e}\mathbf{D}_u^{-\tfrac12}\mathbf{x}_u\bigr)^\dagger
   \vec{\mathcal F}_{v \trianglelefteq e}\mathbf{D}_v^{-\tfrac12}\mathbf{x}_v
   \right).
\end{align*}
Since $\mathbf{L}^{\vec{\mathcal F}}_N$ is Hermitian, the second inner summation can be rewritten as:
\begin{align*}
   - \sum_{\substack{u,v\in e \\ u\neq v}}
   \bigl(\vec{\mathcal F}_{u \trianglelefteq e}\mathbf{D}_u^{-\tfrac12}\mathbf{x}_u\bigr)^\dagger
   \vec{\mathcal F}_{v \trianglelefteq e}\mathbf{D}_v^{-\tfrac12}\mathbf{x}_v =\\
   - \sum_{\substack{u,v\in e \\ u < v}}
   \left(
   (\vec{\mathcal F}_{u \trianglelefteq e}\mathbf{D}_u^{-\tfrac12}\mathbf{x}_u\bigr)^\dagger
   \vec{\mathcal F}_{v \trianglelefteq e}\mathbf{D}_v^{-\tfrac12}\mathbf{x}_v
   +\bigl(\vec{\mathcal F}_{v \trianglelefteq e}\mathbf{D}_v^{-\tfrac12}\mathbf{x}_v\bigr)^\dagger \vec{\mathcal F}_{u \trianglelefteq e}\mathbf{D}_u^{-\tfrac12}\mathbf{x}_u
   \right) =\\
   - \sum_{\substack{u,v\in e \\ u < v}}
   2\,\Re\left[
   \bigl(\vec{\mathcal F}_{u \trianglelefteq e}\mathbf{D}_u^{-\tfrac12}\mathbf{x}_u\bigr)^\dagger
   \vec{\mathcal F}_{v \trianglelefteq e}\mathbf{D}_v^{-\tfrac12}\mathbf{x}_v\right] =\\
   - \sum_{\substack{u,v\in e \\ u \neq v}}
   \Re\left[
   \bigl(\vec{\mathcal F}_{u \trianglelefteq e}\mathbf{D}_u^{-\tfrac12}\mathbf{x}_u\bigr)^\dagger
   \vec{\mathcal F}_{v \trianglelefteq e}\mathbf{D}_v^{-\tfrac12}\mathbf{x}_v\right].
\end{align*}
Substituting back and doubling both terms of the summation, we obtain:
\begin{align*}
\mathcal{E}_N(\mathbf{x})
&= \frac12\sum_{e\in E}\frac{1}{\delta_e}
   \sum_{\substack{u,v\in e \\ u\neq v}}
   \Bigl(
   \|\vec{\mathcal F}_{u \trianglelefteq e}\mathbf{D}_u^{-\tfrac12}\mathbf{x}_u\|_2^2
   + \|\vec{\mathcal F}_{v \trianglelefteq e}\mathbf{D}_v^{-\tfrac12}\mathbf{x}_v\|_2^2
   - 2\,\Re\left[
   \bigl(\vec{\mathcal F}_{u \trianglelefteq e}\mathbf{D}_u^{-\tfrac12}\mathbf{x}_u\bigr)^\dagger
   \vec{\mathcal F}_{v \trianglelefteq e}\mathbf{D}_v^{-\tfrac12}\mathbf{x}_v\right]
   \Bigr).
\end{align*}
Thanks to the identity $\|a-b\|^2 = \|a\|^2+\|b\|^2 - 2\Re(a^\dagger b)$, we conclude:
\[
\mathcal{E}_N(\mathbf{x})
= \frac12\sum_{e\in E}\frac{1}{\delta_e}
   \sum_{\substack{u,v\in e \\ u\neq v}}
   \Bigl\|\vec{\mathcal F}_{u \trianglelefteq e}\mathbf{D}_u^{-\tfrac12}\mathbf{x}_u
   - \vec{\mathcal F}_{v \trianglelefteq e}\mathbf{D}_v^{-\tfrac12}\mathbf{x}_v\Bigr\|_2^2.
\]
Notice that the constraint $u \neq v$ can be dropped from the inner summation w.l.o.g..
\end{proof}

\psd*
\begin{proof}
% Since $\mathbf{L}^{\vec{\mathcal{F}}}_N$ is Hermitian, it can be diagonalized as $U \Lambda U^{*}$ for some $U \in \mathbb{C}^{nd \times nd}$ and $\Lambda \in \mathbb{R}^{nd \times nd}$, where $\Lambda$ is diagonal and real. We have:
% \[
% x^{\dagger}\mathbf{L}^{\vec{\mathcal{F}}}_N x = x^{\dagger} U \Lambda U^{\dagger} x = y^{\dagger} \Lambda y \quad \text{with } y = U^{\dagger}x.
% \]
% Since $\Lambda$ is diagonal, we have:
% \[
% y^{*} \Lambda y = \sum_{u \in V} \lambda_u y_u^2.
% \]
% Thanks to \cref{thm:dirichlet}, the quadratic form $x^{\dagger}\mathbf{L}^{\vec{\mathcal{F}}}_N x$ associated with $\mathbf{L}^{\vec{\mathcal{F}}}_N $ is a sum of squares and, hence, nonnegative. Combined with:
% \[
% x^{\dagger}\mathbf{L}^{\vec{\mathcal{F}}}_N x = \sum_{u \in V} \lambda_u y_u^2,
% \]
% we deduce $\lambda_u \geq 0$ for all $u \in V$, hence $\mathbf{L}^{\vec{\mathcal{F}}}_N$ is positive semidefinite.
This follows directly from the previous theorem.
\end{proof}

\bounded*
\begin{proof}
By definition, we have $\mathbf{L}^{\vec{\mathcal{F}}}_N := \mathbf{I}_{nd} - \mathbf{Q}^{\vec{\mathcal{F}}}_N$, with $\mathbf{Q}^{\vec{\mathcal{F}}}_N  :=  \Dvc^{-\frac12} {\mathbf{B}^{(q)}}^\dagger \mathbf{D}_E^{-1} \mathbf{B}^{(q)} \Dvc^{-\frac12}$.

$\mathbf{Q}^{\vec{\mathcal{F}}}_N$ can be factored as
$$
\mathbf{Q}^{\vec{\mathcal{F}}}_N =
\left(\Dvc^{-\frac12} {\mathbf{B}^{(q)}}^\dagger \mathbf{D}_E^{-\frac12} \right)  \left(\mathbf{D}_E^{-\frac12} \mathbf{B}^{(q)} \Dvc^{-\frac12}\right) = 
\left(\mathbf{D}_E^{-\frac12} \mathbf{B}^{(q)} \Dvc^{-\frac12} \right)^\dagger  \left(\mathbf{D}_E^{-\frac12} \mathbf{B}^{(q)} \Dvc^{-\frac12}\right).
$$
It follows that
$$
\mathbf{x}^\dagger \mathbf{Q}^{\vec{\mathcal{F}}}_N \mathbf{x} =
||\mathbf{x}^\dagger\mathbf{D}_E^{-\frac12} \mathbf{B}^{(q)} \mathbf{D}_V^{-\frac12}\mathbf{x}||^2 \geq 0,
$$
which implies that its spectrum is nonnegative.

Since $\mathbf{L}^{\vec{\mathcal{F}}}_N := \mathbf{I}_{nd} - \mathbf{Q}^{\vec{\mathcal{F}}}_N$, it follows that the spectrum of $\mathbf{L}^{\vec{\mathcal{F}}}_N$ is upper-bounded by 1, which concludes the proof.
\end{proof}

\subsection{Generalization Properties}

\bodnar*
\begin{proof}
In the $2$-uniform case, every hyperedge $e$ contains exactly two nodes (i.e., $\delta_e = 2$).
Consider the general expression of the unnormalized Laplacian given in \cref{eq:laplacian-expanded}. Since the graph has no directions, $\mathcal{S}^{(0)}_{u \trianglelefteq e} = 1$ for all $u \in V, e \in E$, and for any choice of the charge parameter $q$.
As a result, the off-diagonal terms of $\mathbf{L}^{\vec{\mathcal{F}}}$ are real-valued (the diagonal ones always are).

In particular, when $\delta_e = 2$ for all $e \in E$,  $\mathbf{L}^{\vec{\mathcal{F}}}$ reads:
\begin{equation*} 
(\mathbf{L}^{\vec{\mathcal{F}}})_{uv} = 
\begin{cases}
\displaystyle
\tfrac{1}{2}\sum_{e:\,u\in e}
\mathcal{F}_{u \trianglelefteq e}^{\top}
\mathcal{F}_{u \trianglelefteq e}
\in \mathbb{R}^{d \times d}, & u = v, \\[14pt]
\displaystyle
-\tfrac{1}{2}\,
\mathcal{F}_{u \trianglelefteq e}^{\top}
\mathcal{F}_{v \trianglelefteq e}
\in \mathbb{R}^{d \times d}, & u \neq v, \\[6pt]
% \mathbf{0}\in\mathbb{R}^{d\times d}, & \text{otherwise.}
\end{cases}
\end{equation*}

Thus, $\mathbf{L}^{\vec{\mathcal{F}}}$ precisely coincides with the Sheaf Laplacian of~\cite{hansen2020sheafneuralnetworks} up to the multiplicative constant $\tfrac{1}{2}$.

When considering the case of a trivial Sheaf (i.e., when $\mathcal{F}_{v \trianglelefteq e} = 1$), $\mathbf{L}^{\vec{\mathcal{F}}}$ coincides with the definition of the classical graph Laplacian $\mathbf{L} = \mathbf{D} - \mathbf{A}$, where $A$ is the adjacency matrix and $\mathbf{D}$ is the node degree matrix.

Let us note that, in both cases, this constant factor is immaterial in practice, as it can be absorbed by the learnable parameters of the associated neural model.
\end{proof}

\magnet*
\begin{proof}
The \emph{Magnetic Laplacian} proposed by~\cite{zhang2021magnetneuralnetworkdirected} is defined as
\[
\mathbf{L}^{(q)} \coloneqq \mathbf{D}_s - \mathbf{H}^{(q)} 
= \mathbf{D}_s - \mathbf{A}_s \odot \exp\!\bigl(i\,\mathbf{\Theta}^{(q)}\bigr),
\]
where $\mathbf{\Theta}^{(q)}$ denotes the phase matrix defined as
\[
\mathbf{\Theta}^{(q)} \coloneq 2\pi q\,(\mathbf{A} - \mathbf{A}^\top)
\]
and $\mathbf{A}_s$ is the symmetrized adjacency matrix defined as
$$
\mathbf{A}_s := \tfrac{1}{2}(\mathbf{A}+\mathbf{A}^\top)
$$
and $\mathbf{D}_s$ is a diagonal matrix defined as
$$
(\mathbf{D}_s)_{uu} := \sum_{v\in V} (\mathbf{A}_s)_{uv} \text{ for all $u \in V$}.
$$
Entry-wise, $\mathbf{H}^{(q)}$ can be written as:
\[
\mathbf{H}^{(q)}_{uv} =
\begin{cases}
\tfrac{1}{2} e^{\,2\pi i q} & (u,v) \in E \\[3pt]
\tfrac{1}{2} e^{-\,2\pi i q} & (v,u) \in E \\[3pt]
1 & \{u,v\} \in E\\[3pt]
0 & \text{otherwise}.
\end{cases}
\]
In the directed, 2-uniform case, every hyperedge $e$ contains exactly two nodes ($\delta_e = 2$). For every $e \in E$, the product $(\mathbf{\mathcal{S}}^{(q)}_{u\trianglelefteq e})^{\dagger} \mathbf{\mathcal{S}}^{(q)}_{v\trianglelefteq e}$ can take one of the following three values:
\begin{enumerate}
\item Undirected edge $e = \{u,v\}$:
\[
\mathbf{\mathcal{S}}^{(q)}_{u\trianglelefteq e}=\mathbf{\mathcal{S}}^{(q)}_{v\trianglelefteq e}=1
\quad\Longrightarrow\quad
\big(\mathbf{\mathcal{S}}^{(q)}_{u\trianglelefteq e}\big)^{\dagger}
\mathbf{\mathcal{S}}^{(q)}_{v\trianglelefteq e}=1.
\]

\item Directed edge $e= (u,v)$:
\[
\mathbf{\mathcal{S}}^{(q)}_{u\trianglelefteq e}=e^{-2 \pi i q},\qquad
\mathbf{\mathcal{S}}^{(q)}_{v\trianglelefteq e}=1
\quad\Longrightarrow\quad
\big(\mathbf{\mathcal{S}}^{(q)}_{u\trianglelefteq e}\big)^{\dagger}
\mathbf{\mathcal{S}}^{(q)}_{v\trianglelefteq e}=e^{+2\pi i q}.
\]

\item Directed edge $e = (v,u)$:
\[
\mathbf{\mathcal{S}}^{(q)}_{u\trianglelefteq e}=1,\qquad
\mathbf{\mathcal{S}}^{(q)}_{v\trianglelefteq e}=e^{-2\pi i q}
\quad\Longrightarrow\quad
\big(\mathbf{\mathcal{S}}^{(q)}_{u\trianglelefteq e}\big)^{\dagger}
\mathbf{\mathcal{S}}^{(q)}_{v\trianglelefteq e}=e^{-2 \pi i q}.
\]
\end{enumerate}

Letting (w.l.o.g., as the restriction maps are learnable)
\[
\begin{cases}
\mathcal{F}_{v \trianglelefteq e} := \sqrt{2}, \quad \mathcal{F}_{u \trianglelefteq e} := \sqrt{2} & \text{if } e = \{u,v\}, \\[6pt]
\mathcal{F}_{v \trianglelefteq e} := 1, \quad \mathcal{F}_{u \trianglelefteq e} = 1 & \text{if } e = (u,v) \ \text{or}\ e = (v,u).
\end{cases}
\]
we have:
\[
(\mathbf{Q}^{\vec{\mathcal{F}}})_{uv} =
\begin{cases}
\begin{aligned}[t]
\frac{1}{2}\,
\vec{\mathcal{F}}_{u\trianglelefteq e}^{\dagger}\,
\vec{\mathcal{F}}_{v\trianglelefteq e}
&= \frac{1}{2}\,
   e^{+\,2\pi i q}\,
   \mathcal{F}_{u\trianglelefteq e}^{\top}\mathcal{F}_{v\trianglelefteq e} = \frac{1}{2}e^{+\,2\pi i q},
\end{aligned}
& \text{if } e=(u,v), \\[12pt]
\begin{aligned}[t]
\frac{1}{2}\,
\vec{\mathcal{F}}_{u\trianglelefteq e}^{\dagger}\,
\vec{\mathcal{F}}_{v\trianglelefteq e}
&= \frac{1}{2}\,
   e^{-\,2\pi i q}\,
   \mathcal{F}_{u\trianglelefteq e}^{\top}\mathcal{F}_{v\trianglelefteq e} = \frac{1}{2}\,
   e^{-\,2\pi i q},
\end{aligned}
& \text{if } e=(v,u), \\[12pt]
\begin{aligned}[t]
\frac{1}{2}\,
\vec{\mathcal{F}}_{u\trianglelefteq e}^{\dagger}\,
\vec{\mathcal{F}}_{v\trianglelefteq e}
&= \frac{1}{2}\,
   \mathcal{F}_{u\trianglelefteq e}^{\top}\mathcal{F}_{v\trianglelefteq e} = 1,
\end{aligned}
& \text{if } e=\{u,v\}.
\end{cases}
\]

Hence, by construction, we have:
\[
\mathbf{Q}^{\vec{\mathcal{F}}} ={\mathbf{B}^{(q)}}^{\dagger}\, \mathbf{D}_E^{-1}\, \mathbf{B}^{(q)} \;=\; \mathbf{H}^{(q)}, \text{ with } \Dvc \;=\; \mathbf{D}_s.
\]
This implies:
\[
\Lc =  \mathbf{D}_s - \mathbf{H}^{(q)} = \mathbf{L}^{(q)}.
\]

Lastly, noticing that, by construction, the Sign-Magnetic Laplacian proposed in \cite{fiorini2023sigmanetlaplacianrule} coincides with the Magnetic Laplacian when $q = \frac{1}{4}$, we conclude that our operator also generalizes the former.
\end{proof}

\zhou*
\begin{proof}  
In the unit-weight case, the Laplacian matrix proposed by \cite{zhou} for undirected hypergraphs is defined as follows:
$$
\mathbf{\Delta} := \mathbf{I} - \mathbf{Q}_N \qquad \text{ with } \mathbf{Q}_N := \Dvc^{-\frac12} \mathbf{B} \mathbf{D}_E^{-1} \mathbf{B}^\top \Dvc^{-\frac12}.
$$
Since any undirected hypergraph be regarded as a special case of a directed hypergraph in which every hyperedge consists solely of tail nodes (or, equivalently, solely of head nodes), as shown in \cref{eq:interaction-types}, in our proposed Laplacian matrix $\Lc_N$ each product of two restriction maps reduces to a real weight of $1$, therefore contributing only to the real part of the operator.
In particular, for a trivial sheaf where $\mathcal{F}_{v \trianglelefteq e} = 1$, the incidence matrix $\mathbf{B}^{(q)}$ in \cref{eq:incidence-expanded} reduces to the transpose of binary incidence matrix $B$ of \cite{zhou}.
\end{proof}

\gedi*
\begin{proof}
Let's consider a special case of a trivial sheaf (i.e. $\mathcal{F}_{u \trianglelefteq e} = 1$). By setting $q = \frac{1}{4}$ we have:
\[
    S^{(0.25)}_{u \trianglelefteq e} = 
    \begin{cases}
    1 & \text{if } u \in H(e) \quad \text{(head set)} \\
    -i & \text{if } u \in T(e) \quad \text{(tail set)} \\
    0 & \text{otherwise}
    \end{cases}
\]
Now, for each pair $u,v$ belonging to the same hyperedge $e$:
\[
\vec{\mathcal{F}}_{u \trianglelefteq e}^\dagger \vec{\mathcal{F}}_{v \trianglelefteq e}
=\big(S^{(0.25)}_{u\trianglelefteq e}\big)^{\dagger}S^{(0.25)}_{v\trianglelefteq e}\,
\]
Whose contribution, according to the four cases in \cref{eq:interaction-types}, is given by: 
\[
\vec{\mathcal{F}}_{u \trianglelefteq e}^\dagger \vec{\mathcal{F}}_{v\trianglelefteq e}
=
\begin{cases}
1, & u,v \in H(e),\\
1, & u,v \in T(e),\\
i , & u\in T(e),\ v\in H(e),\\
-i , & u\in H(e),\ v\in T(e).
\end{cases}
\]
Our Normalized Directed Sheaf Hypergraph Laplacian $\mathbf{L}^{\vec{\mathcal{F}}}_N$, component-wise reads:
\[
(\mathbf{L}^{\vec{\mathcal{F}}}_N)_{uv} = 
    \left\{
    \begin{array}{lr}
    \displaystyle

    \mathbf{I}_{d} - \mathbf{D}_u^{-1}\sum_{e: u \in e} \frac{1}{\delta_{e}}{\mathcal{F}}_{u \trianglelefteq e}^{\top}{\mathcal{F}}_{u \trianglelefteq e} 
    & u = v
    
    \\
    
    \displaystyle
    - \mathbf{D}_u^{-\frac{1}{2}} \big (\sum_{e: u,v \in e} \frac{1}{\delta_{e}} \vec{\mathcal{F}}_{u \trianglelefteq e}^{\dagger} \vec{\mathcal{F}}_{v \trianglelefteq e} \big ) \mathbf{D}_v^{-\frac{1}{2}} & u \neq v.
    \end{array} 
    \right.
\]

Which reduces, in the considered scalar special case to:
\begin{eqnarray}
(\mathbf{L}^{\vec{\mathcal{F}}}_N)_{uv} =
\left\{
\begin{array}{ll}
\displaystyle
1 - \sum_{e:\,u\in e}\frac{1}{\mathbf{D}_u\,\delta_e},
& u = v, \\[6pt]
\displaystyle
-\!\!\!\!\!\sum_{\begin{subarray}{c}
e \in E\\
u,v \in H(e)\\
\lor\; u,v \in T(e)
\end{subarray}}
\!\!\!\!\!
\frac{1}{\delta_e}
-
i\left(
\sum_{\begin{subarray}{c}
e \in E\\
u \in T(e)\\
\land v \in H(e)
\end{subarray}}
\!\!\frac{1}{\delta_e}
-\!\!\!\!
\sum_{\begin{subarray}{c}
e \in E\\
u \in H(e)\\
\land v \in T(e)
\end{subarray}}
\!\!\frac{1}{\delta_e}
\right)
\frac{1}{\sqrt{\mathbf{D}_u}\sqrt{\mathbf{D}_v}},
& u \neq v.
\end{array}
\right.
\end{eqnarray}

Such an expression coincides with the definition of the Generalized Directed Laplacian when considering $\mathbf{W}=\mathbf{I}$.%\footnote{The sign difference in the imaginary part, relative to the original formulation, is due solely \cite{fiorini2024let} calling "head" what we call "tail" here and vice versa, possibly due to a typo.
%Such a choice induces an overall sign flip in the imaginary component of the operator; 
% To the best of our understanding, all the statements and proofs in \cite{fiorini2024let} are valid with such a swap.}
\end{proof}

\section{Extended experimental evaluation}

In this section, we include further experiments and details that did not make the cut in the main paper due to space limits.
This includes:
\begin{itemize}
    \item The optimal value of the charge parameter $q$ found for \approach and \light during the hyperparameters optimization process.
    \item The impact of the stalk dimension $d$ and the number of layers on the method's performance.
    \item The complete results on 12 real-world datasets.
\end{itemize}

\subsection{Impact of charge parameter}
The charge parameter $q$ controls how much each directed hyperedge contributes to the \emph{real} and \emph{imaginary} parts of the Directed Sheaf Hypergraph Laplacian. Larger values of $q$ place more directional information in the imaginary component, whereas smaller values reduce the directional contribution, emphasizing orientation-agnostic interactions in the real part. Because the dataset differ in how informative directionality is, the optimal $q$ is inherently data-dependent.
In practice, a careful tuning of it is needed $q$ to select the value that yields the best performance, allowing either a partial or a full contribution of directional information to be encoded as needed.
\cref{tab:q_values_table} reports the values chosen by our hyperparameter tuning procedure. As one can seen, for most datasets the hyperparameter tuning procedure sets a relatively high importance to directional information for each dataset, particularly for \texttt{Telegram}, \texttt{Roman-empire} and Synthetic datasets.

\begin{table*}[ht]
\caption{Optimal $q$ values for \approach and \light across all real-world and synthetic datasets found by hyperparameter tuning.}
\label{tab:q_values_table}
\centering
\resizebox{\linewidth}{!}{
\begin{tabular}{lcccccc}
\toprule
\textbf{Method} & \textbf{Roman-empire} & \textbf{Squirrel} & \textbf{email-EU} & \textbf{Telegram} & \textbf{Chameleon} & \textbf{email-Enron} \\
\midrule
\approach & ---  & 0.05 & 0.25 & 0.25 & 0.20 & 0.05 \\
\light    & 0.20 & 0.05 & 0.20 & 0.20 & 0.15 & 0.15 \\
\bottomrule
\end{tabular}
}
\resizebox{\linewidth}{!}{
\begin{tabular}{lcccccc}
\toprule
\textbf{Method} & \textbf{Cornell} & \textbf{Wisconsin} & \textbf{Amazon-ratings} & \textbf{Texas} & \textbf{Citeseer} & \textbf{Cora} \\
\midrule
\approach & 0.25 & 0.25 & ---  & 0.25 & 0.00 & 0.00 \\
\light    & 0.15 & 0.25 & 0.00 & 0.15 & 0.00 & 0.00 \\
\bottomrule
\end{tabular}
}
\begin{tabular}{lccc}
\toprule
\textbf{Method} & \textbf{$I_o=10$} & \textbf{$I_o=30$} & \textbf{$I_o=50$} \\
\midrule
\approach & 0.25 & 0.10 & 0.10 \\
\light    & 0.10 & 0.10 & 0.10 \\
\bottomrule
\end{tabular}
\end{table*}

\subsection{Impact of stalk dimension and number of layers}
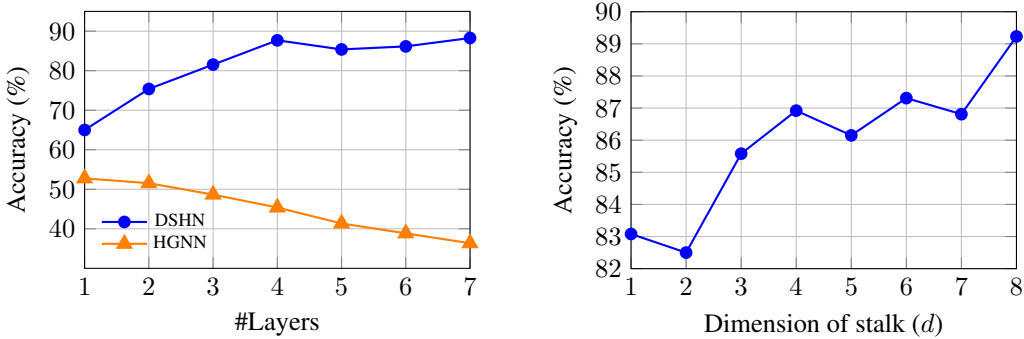
\begin{figure}[h]
    \centering
    % First subfigure
    \begin{subfigure}{0.48\linewidth}
        \centering
        \begin{tikzpicture}
        \begin{axis}[
            width=\linewidth,
            height=5cm,
            xlabel={\#Layers},
            ylabel={Accuracy (\%)},
            ymin=30, ymax=95,
            xmin=1, xmax=7,
            xtick={1,...,7},
            ytick={40,50,60,70,80,90},
            legend style={
                at={(0.02,0.02)}, % bottom-left
                anchor=south west,
                draw=none,
                fill=none,
                font=\scriptsize,
                row sep=-2pt
            },
            grid=major,
        ]
        % \approach data (exact values)
        \addplot[blue, thick, mark=*, mark size =2pt] coordinates {
            (1,65.00) (2,75.38) (3,81.54) (4,87.69) (5,85.38) (6,86.15) (7,88.27)
        };
        \addlegendentry{\approach}
        
        % HGNN data (exact values, triangles)
        \addplot[orange, thick, mark=triangle*, mark size=3pt] coordinates {
            (1,52.76) (2,51.54) (3,48.65) (4,45.38) (5,41.35) (6,38.85) (7,36.35)
        };
        \addlegendentry{HGNN}
        
        \end{axis}
        \end{tikzpicture}
        \caption{Accuracy of \approach and HGNN as the number of layers increases.}
        \label{fig:influence_num_layers}
    \end{subfigure}
    \hfill
    % Second subfigure
    \begin{subfigure}{0.48\linewidth}
        \centering
        \begin{tikzpicture}
        \begin{axis}[
            width=\linewidth,
            height=5cm,
            xlabel={Dimension of stalk ($d$)},
            ylabel={Accuracy (\%)},
            ymin=82, ymax=90,
            xmin=1, xmax=8,
            xtick={1,...,8},
            ytick={82,83,84,85,86,87,88,89,90},
            grid=major,
        ]
        % \approach stalk dimension data (exact values)
        \addplot[blue, thick, mark=*] coordinates {
            (1,83.08) (2,82.50) (3,85.58) (4,86.92) (5,86.15) (6,87.31) (7,86.81) (8,89.23)
        };
        
        \end{axis}
        \end{tikzpicture}
        \caption{Accuracy of \approach as the stalk dimension $d$ increases.}
        \label{fig:influence_stalk_dim}
    \end{subfigure}
    
    \caption{Influence of architectural parameters on accuracy. 
    (a) Effect of the number of layers on \approach and HGNN. 
    (b) Effect of stalk dimension $d$ on \approach.}
    \label{fig:influence_parameters}
\end{figure}

As noted in \cref{sec:introduction}, standard HGNNs are prone to \emph{oversmoothing}: as network depth increases, node representations become indistinguishable and accuracy degrades. In \cref{fig:influence_parameters}, we study how depth and the stalk dimension \(d\) affect the accuracy of \approach.
\approach shows no signs of oversmoothing, as accuracy \emph{improves} as we add layers. Performance also increases with a higher stalk dimension \(d\), underscoring the additional expressive power associated to cellular sheaves. This stands in clear contrast to HGNN, whose accuracy steadily deteriorates with depth. This is in line with the observations in \cite{bodnar2023neuralsheafdiffusiontopological} for graphs: leveraging our Directed Sheaf Hypergraph Laplacian, built with \(d\times d\) restriction maps to transport features between nodes and hyperedges, enriches local variability rather than collapsing it. By projecting node features onto hyperedges (and back), the model retains discriminative power across neighborhoods.
\subsection{Extended Results}
\label{app:extended_results}
%
% \paragraph{Node Classification}
In this subsection, we report the complete table of results for this work, which were not reported in the main just due to space limitations. As can be observed from \cref{tab:results_real_world_complete}, \approach, and \light consistently outperform the baselines taken from both the directed and undirected hypergraph learning literature on 10 out of 12 considered real-world datasets. 

\begin{table*}[ht]
\centering
\caption{Mean accuracy and standard deviation on node classification datasets (test accuracy $\pm$ std). For each dataset, the best result is shown in \textbf{bold}, and the second-best is \second{underlined}.}
\label{tab:results_real_world_complete}
\resizebox{\linewidth}{!}{
\begin{tabular}{lcccccc}
\toprule
\textbf{Method} & 
\textbf{Roman-empire} & \textbf{Squirrel} & \textbf{email-EU} & \textbf{Telegram} & \textbf{Chameleon} & \textbf{email-Enron} \\
\midrule
HGNN              & $38.44 \pm 0.44$ & $35.47 \pm 1.44$ & $48.91 \pm 3.11$ & $51.73 \pm 3.38$ & $39.98 \pm 2.28$ & $52.85 \pm 7.27$ \\
HNHN              & $46.07 \pm 1.22$ & $35.62 \pm 1.30$ & $29.68 \pm 1.68$ & $38.22 \pm 6.95$ & $35.81 \pm 3.23$ & $18.64 \pm 6.90$ \\
UniGCNII          & $78.89 \pm 0.51$ & $38.28 \pm 2.56$ & $44.98 \pm 2.69$ & $51.73 \pm 5.05$ & $39.85 \pm 3.19$ & $47.43 \pm 7.47$ \\
LEGCN             & $65.60 \pm 0.41$ & $39.18 \pm 1.54$ & $32.91 \pm 1.83$ & $45.38 \pm 4.23$ & $39.29 \pm 2.04$ & $37.03 \pm 7.16$ \\
HyperND           & $68.31 \pm 0.69$ & $40.13 \pm 1.85$ & $32.79 \pm 2.90$ & $44.62 \pm 5.49$ & $44.95 \pm 3.20$ & $38.11 \pm 7.69$ \\
AllDeepSets       & $81.79 \pm 0.72$ & $40.69 \pm 1.90$ & $37.37 \pm 6.29$ & $49.19 \pm 6.73$ & $42.97 \pm 3.60$ & $37.29 \pm 7.90$ \\
AllSetTransformer & $83.53 \pm 0.64$ & $40.53 \pm 1.33$ & $38.26 \pm 3.57$ & $66.92 \pm 4.36$ & $43.85 \pm 5.42$ & $63.78 \pm 3.66$ \\
ED\mbox{-}HNN     & $83.82 \pm 0.31$ & $39.85 \pm 1.79$ & $68.91 \pm 4.00$ & $60.38 \pm 3.86$ & $44.67 \pm 2.33$ & $51.35 \pm 6.04$ \\
SheafHyperGNN     & $74.50 \pm 0.57$ & $42.01 \pm 1.11$ & $52.78 \pm 9.13$ & $70.00 \pm 5.32$ & $41.06 \pm 4.94$ & $63.51 \pm 5.95$ \\
PhenomNN          & $71.22 \pm 0.45$ & $39.45 \pm 2.19$ & $37.69 \pm 4.40$ & $47.69 \pm 6.59$ & $43.62 \pm 4.29$ & $47.02 \pm 6.75$ \\
\midrule
GeDi\mbox{-}HNN   & \second{$83.87 \pm 0.63$} & $43.02 \pm 3.00$ & $52.31 \pm 2.84$ & $77.12 \pm 4.82$ & $39.29 \pm 2.04$ & $50.54 \pm 5.80$ \\
DHGNN             & $77.58 \pm 0.54$ & $39.85 \pm 1.79$ & $32.35 \pm 2.93$ & $79.62 \pm 5.78$ & $44.08 \pm 4.11$ & $42.16 \pm 8.04$ \\
DHGNN (w/ emb.)   & $22.50 \pm 0.81$ & $40.33 \pm 1.42$ & $55.10 \pm 3.48$ & $80.58 \pm 3.89$ & $40.85 \pm 2.76$ & $58.38 \pm 7.57$ \\
\midrule
\textbf{\approach}         & OOM              & \second{$43.55 \pm 2.87$} & \second{$78.62 \pm  2.50$} & \first{$88.65 \pm 5.54$} & \first{$47.02 \pm 4.35$} & \second{$75.68 \pm 3.42$} \\
\textbf{\light}    & \first{$89.24 \pm 0.57$} & \first{$44.09 \pm 2.36$} & \first{$82.67 \pm 1.29$} & \second{$81.15 \pm 4.19$} & \second{$46.50 \pm 4.09$} & \first{$76.76 \pm 2.48$} \\

\bottomrule
\end{tabular}
}
\resizebox{\linewidth}{!}{
\begin{tabular}{lcccccc}
\toprule
\textbf{Method} & \textbf{Cornell} & \textbf{Wisconsin} & \textbf{Amazon-ratings} & \textbf{Texas} & \textbf{Citeseer} & \textbf{Cora} \\
\midrule
HGNN              & $43.51 \pm 6.44$ & $51.56 \pm 6.68$ & $46.20 \pm 0.45$ & $52.77 \pm 7.48$ & $76.02 \pm 0.81$ & $87.25 \pm 1.01$ \\
HNHN              & $43.51 \pm 6.09$ & $49.60 \pm 4.96$ & $42.29 \pm 0.34$ & $58.11 \pm 3.87$ & $71.24 \pm 0.66$ & $78.16 \pm 0.98$ \\
UniGCNII          & $73.24 \pm 5.19$ & $86.86 \pm 4.30$ & $49.12 \pm 0.46$ & $81.35 \pm 5.33$ & $77.30 \pm 1.15$ & $87.53 \pm 1.06$ \\
LEGCN             & $75.14 \pm 5.51$ & $84.71 \pm 4.00$ & $47.02 \pm 0.59$ & $81.35 \pm 4.26$ & $72.62 \pm 1.09$ & $74.96 \pm 0.94$ \\
HyperND           & $75.14 \pm 5.38$ & $86.67 \pm 5.02$ & $47.33 \pm 0.51$ & \second{$83.51 \pm 5.19$} & $75.21 \pm 1.37$ & $78.48 \pm 1.02$ \\
AllDeepSets       & $77.83 \pm 3.78$ & \second{$87.84 \pm 3.69$} & $51.91 \pm 0.68$ & $82.76 \pm 5.74$ & $75.78 \pm 0.94$ & $86.86 \pm 0.85$ \\
AllSetTransformer & $75.94 \pm 2.97$ & $86.27 \pm 3.92$ & $52.28 \pm 0.67$ & $82.76 \pm 5.07$ & $75.61 \pm 1.44$ & $86.73 \pm 1.13$ \\
ED\mbox{-}HNN     & $76.49 \pm 4.53$ & $85.09 \pm 4.89$ & $51.58 \pm 0.53$ & $80.00 \pm 5.05$ & $74.95 \pm 1.27$ & $86.94 \pm 1.25$ \\
SheafHyperGNN     & $74.59 \pm 4.39$ & $85.29 \pm 4.74$ & $48.90 \pm 0.59$ & $80.00 \pm 2.48$ & $77.21 \pm 1.44$ & $87.15 \pm 0.64$ \\
PhenomNN          & $72.16 \pm 4.19$ & $80.58 \pm 6.10$ & $48.81 \pm 0.37$ & $81.49 \pm 4.95$ & $77.21 \pm 1.32$ & \first{$88.12 \pm 0.86$} \\
\midrule
GeDi\mbox{-}HNN   & \second{$78.37 \pm 3.19$} & $87.45 \pm 3.41$ & $49.30 \pm 0.52$ & $82.55 \pm 4.64$ & $75.94 \pm 0.95$ & $85.16 \pm 0.94$ \\
DHGNN             & $77.30 \pm 4.05$ & $87.45 \pm 3.84$& \second{$52.48 \pm 0.50$} & $83.24 \pm 5.64$ & $74.67 \pm 1.24$ & $83.16 \pm 1.33$ \\
DHGNN (w/ emb.)   & $51.08 \pm 4.43$ & $59.80 \pm 5.63$ & \first{$53.64 \pm 0.52$} & $63.51 \pm 9.84$ & $56.78 \pm 1.32$ & $73.12 \pm 1.04$ \\
\midrule
\textbf{\approach}        & \first{$79.19 \pm 4.37$} & \first{$88.63 \pm 3.49$} & OOM & \first{$83.78 \pm 5.13$} & \second{$77.39 \pm 1.04$} & $87.84 \pm 0.90$ \\
% Optimal $q$       & 0.25 & 0.25 & --- & 0.25 & 0.0 & 0.0 \\
% \midrule
\textbf{\light}   & \first{$79.19 \pm 3.20$} & $87.25 \pm  4.90$ & $50.94 \pm 0.68$ & $82.43 \pm 5.44$ & \first{$77.45 \pm 0.74$} & \second{$88.02 \pm 1.11$}\\
% Optimal $q$ & 0.15 & 0.25 & 0.0 & 0.15 & 0.0 & 0.0 \\
\bottomrule
\end{tabular} }
\end{table*}
%
% \paragraph{Hyperedge Classification}
Additionally, we evaluate our method on two real-world directed hypergraph dataset for molecular reaction reframed as a hyperedge classification task, results are provided in \cref{tab:f1_scores}. These datasets are the result of the merging of data from different sources such as \cite{ORD, reizman2016suzuki, LugoMartinez2021Hypergraphlet} and are built inspired by \cite{Restrepo2024Spaces}, which proposes a novel way of modeling molecular reactions through directed hypergraphs. Dataset-1 contains $100{,}523$ nodes and $50{,}016$ hyperedges, with a total of 10 classes. Dataset-2 contains 956 nodes and $3{,}021$ hyperedges to classify among 6 different classes. These datasets consist of inherently directional hyperedges as they contain the molecular reactions expressed as set of reagents (the tail set) and set of products (the head set) composing a molecular reaction. The nodes' features are built based on Morgan Fingerprints \citep{rogers2010extended}, which are one of the most widely used molecular descriptors. We employ the F1-score metric since the data has an imbalanced amount of samples for each class as shown in \cref{app:hyperedge}.

As shown in \cref{tab:f1_scores}, \approach consistently outperforms all competing methods from both the undirected and directed hypergraph learning literature. On Molecular-1, it achieves an F1-score of 82.32\%, improving upon the strongest baseline, GeDi-HNN, by 1.98\%. On Molecular-2, \approach attains 89.09\%, exceeding AllSetTransformer by a relative margin of 1.37\%.

\begin{table*}[h]
\centering
\caption{Mean F1-score and standard deviation for hyperedge classification on two molecular reaction datasets (test F1-score $\pm$ std). The best score is shown in \textbf{bold}, and the second-best is \underline{underlined}.}
\label{tab:f1_scores}
\resizebox{0.5\textwidth}{!}{
\begin{tabular}{lcc}
\toprule
\textbf{Method} & \textbf{Molecular-1} & \textbf{Molecular-2} \\
\midrule
HGNN                & 69.38 $\pm$ 0.48       & 81.40 $\pm$ 2.68 \\
HNHN                & 32.27 $\pm$ 1.30       & 45.69 $\pm$ 7.48 \\
UniGCNII            & 72.00 $\pm$ 0.59       & 85.61 $\pm$ 2.63 \\
LEGCN               & OOM                     & 84.75 $\pm$ 2.68 \\
HyperND             & 44.16 $\pm$ 1.27       & 82.86 $\pm$ 3.17 \\
AllDeepSets         & 79.17 $\pm$ 0.53       & 85.78 $\pm$ 3.01 \\
AllSetTransformer   & 79.24 $\pm$ 1.08       & \underline{87.89 $\pm$ 2.87} \\
ED\text{-}HNN       & 66.37 $\pm$ 2.62       & 87.05 $\pm$ 1.96 \\
SheafHyperGNN       & 57.99 $\pm$ 2.75       & 80.25 $\pm$ 2.20 \\
PhenomNN            & 47.71 $\pm$ 2.90       & 86.27 $\pm$ 2.40 \\
\midrule
GeDi\text{-}HNN     & \underline{80.72 $\pm$ 0.78} & 85.64 $\pm$ 2.42 \\
DHGNN               & OOM                     & 85.93 $\pm$ 3.49 \\
\midrule
\approach           & OOM                     & \textbf{89.09 $\pm$ 3.08} \\
\light              & \textbf{82.32 $\pm$ 0.56}        & 86.52 $\pm$ 2.68 \\
\bottomrule
\end{tabular}
}
\end{table*}

\section{Implementation details} \label{app:implementation-details}
We provide additional details regarding the implementation of our models, with a particular emphasis on the computational complexity of \approach and \light and the architectural choices that contribute to their stability and expressiveness. 

\subsection{Computational Complexity}\label{app:computational-complexity}

\paragraph{Comparison between \approach and \light}

\cref{tab:flops} presents a comparative analysis of \approach and \light, across various datasets, measuring their performance in terms of average FLOPS per epoch and average step time. The results are averaged over 10 runs. Over all the 12 datasets, \light always appears to be more efficient, consistently requiring fewer computational resources while maintaining faster processing times. By applying the aforementioned detachment operation through backpropagation, \light achieves similar and sometimes better results, as can be seen from \cref{tab:results_real_world}.

\begin{table*}[h]
\centering
\caption{\approach vs \light -- FLOPS and Step Time (in ms) Analysis Across Different Datasets (Mean $\pm$ Standard Deviation)}
\label{tab:flops}
\resizebox{\textwidth}{!}{%
\begin{tabular}{lcc|cc}
\toprule
& \multicolumn{2}{c}{Avg FLOPs/epoch$(\downarrow)$} & \multicolumn{2}{c}{Avg step time $(\downarrow)$} \\
\midrule
Dataset & \approach & \light & \approach & \light \\
\midrule
Cora & 267,070,765,386 $\pm$ 0 & 196,828,716,921 $\pm$ 3,250 & 2635.02 $\pm$ 112.51 & 973.18 $\pm$ 164.18 \\
Citeseer & 415,705,637,192 $\pm$ 0 & 310,747,699,339 $\pm$ 5,239 & 2631.83 $\pm$ 146.54 & 958.34 $\pm$ 159.00 \\
email-Enron & 962,025,184 $\pm$ 172 & 696,069,022 $\pm$ 364 & 2559.10 $\pm$ 115.97 & 932.18 $\pm$ 152.09 \\
email-EU & 35,930,593,693 $\pm$ 1,176 & 25,798,032,763 $\pm$ 1,183 & 4170.54 $\pm$ 105.64 & 1018.13 $\pm$ 155.71 \\
Telegram & 2,628,033,910 $\pm$ 0 & 1,858,200,422 $\pm$ 0 & 2702.33 $\pm$ 150.89 & 965.40 $\pm$ 161.12 \\
Cornell & 2,201,584,340 $\pm$ 220 & 1,851,871,460 $\pm$ 220 & 2467.84 $\pm$ 132.16 & 886.07 $\pm$ 164.73 \\
Texas & 2,228,554,459 $\pm$ 0 & 1,876,832,187 $\pm$ 0 & 2480.09 $\pm$ 130.18 & 888.02 $\pm$ 164.15 \\
Wisconsin & 3,684,554,183 $\pm$ 201 & 3,035,436,853 $\pm$ 454 & 2547.95 $\pm$ 116.24 & 923.18 $\pm$ 161.33 \\
Chameleon & 34,986,115,734 $\pm$ 123 & 27,033,570,342 $\pm$ 123 & 2629.45 $\pm$ 132.52 & 959.08 $\pm$ 155.75 \\
Squirrel & 189,607,210,489 $\pm$ 5,694 & 140,531,198,787 $\pm$ 3,557 & 3870.93 $\pm$ 119.79 & 1046.00 $\pm$ 169.09\\
Roman-empire & OOM & 12,898,147,996,391 $\pm$ 43,606 & OOM & 1050.30 $\pm$ 152.04  \\
Amazon-ratings & OOM & 15,061,770,374,298 $\pm$ 0 & OOM & 1080.26 $\pm$ 159.91 \\
\bottomrule
\end{tabular}%
}
\end{table*}

\paragraph{Comparison between \approach and other models}

\begin{figure}[h]
    \centering
    \includegraphics[width=0.6\linewidth]{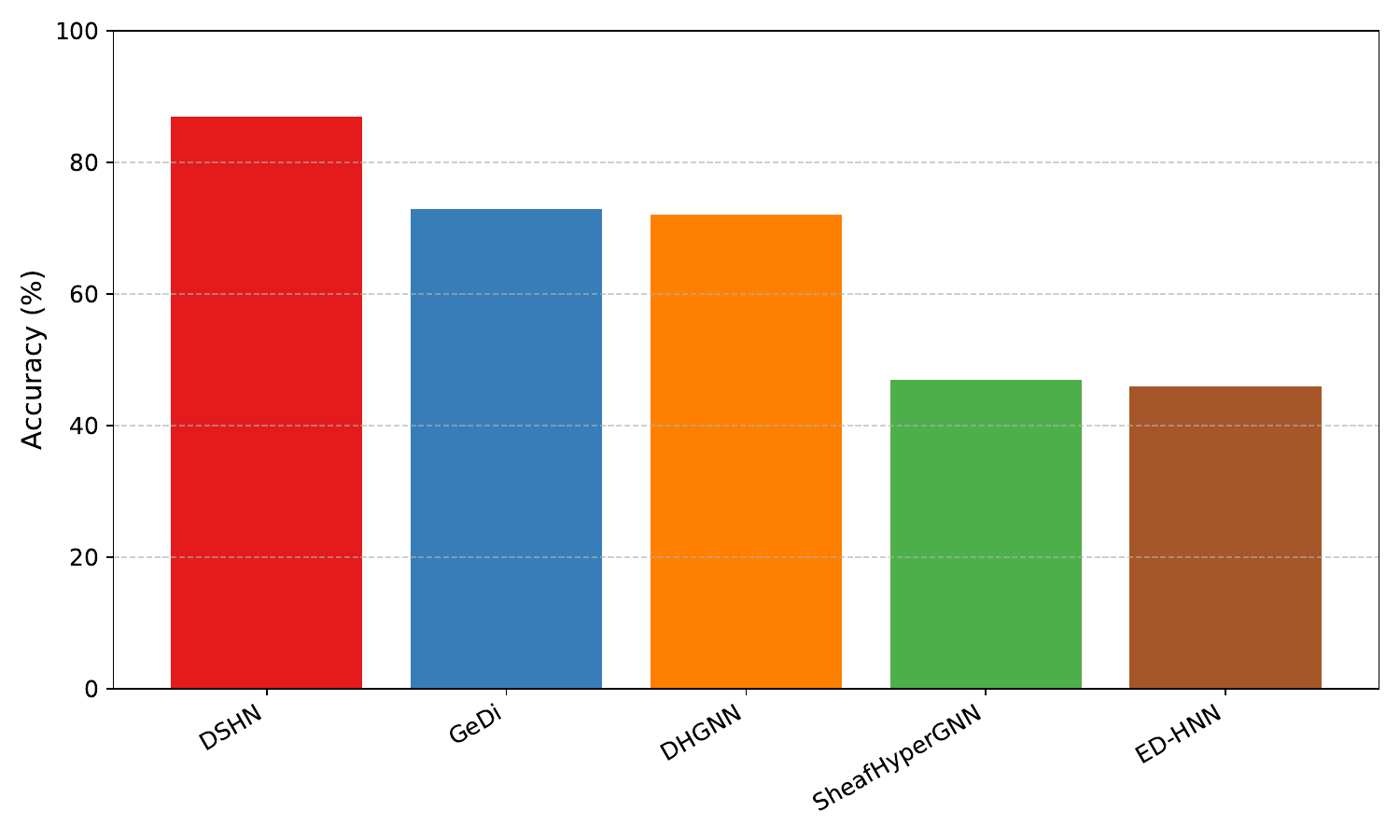}
    \caption{Comparison between models under the same number of parameters ($\sim 80k$) on the Telegram dataset.}
    \label{fig:comparison_models_parameters}
\end{figure}
    
\cref{fig:comparison_models_parameters} reports the average test accuracy of five representative models under approximately the same parameter budget. The results indicate that model size alone does not explain the performance of \approach. For instance, although SheafHyperGNN and ED-HNN have a comparable number of parameters, their accuracy is significantly lower, being these undirected methods. In contrast, \approach achieves an improvement of about 8\% over the strongest directed baselines, despite having the same number of parameters thanks to the expressive power associated to complex-valued and directional restriction maps. On the other hand, \cref{tab:flops_methods_full} shows a comparison of \light with other models. Both \light and SheafHyperGNN generally incur higher computational costs than traditional hypergraph neural networks, although achieving, at least in the case of \light, a substantially better accuracy. This overhead stems directly from the requirement to learn and apply restriction maps at every node-hyperedge incidence as detailed in the next paragraph.

\begin{table*}[h]
\centering
\caption{FLOPs and Parameter Count Across Datasets and Methods.}
\label{tab:flops_methods_full}
\resizebox{\textwidth}{!}{%
\begin{tabular}{l|cc|cc|cc|cc}
\toprule
Dataset 
& \multicolumn{2}{c|}{ED-HNN} 
& \multicolumn{2}{c|}{SheafHyperGNN} 
& \multicolumn{2}{c|}{DHGNN}
& \multicolumn{2}{c}{\light} \\
\midrule
& FLOPs & \#Params
& FLOPs & \#Params
& FLOPs & \#Params
& FLOPs & \#Params \\
\midrule

Cora
& $3\times10^{9}$  & $125{,}959$
& $2\times10^{11}$ & $404{,}576$
& $7\times10^{10}$ & $4{,}077{,}051$
& $2.0\times10^{11}$ & $437{,}447$ \\

Citeseer
& $5\times10^{9}$  & $271{,}174$
& $3\times10^{11}$ & $985{,}440$
& $3\times10^{11}$ & $12{,}751{,}318$
& $3.1\times10^{11}$ & $1{,}018{,}502$ \\

email-Enron
& $4\times10^{8}$  & $34{,}311$
& $6\times10^{8}$  & $37{,}984$
& $2\times10^{7}$  & $13{,}334$
& $7.0\times10^{8}$ & $70{,}855$ \\

email-EU
& $5\times10^{9}$  & $34{,}506$
& $3\times10^{10}$ & $38{,}752$
& $5\times10^{8}$  & $14{,}372$
& $2.6\times10^{10}$ & $71{,}050$ \\

Telegram
& $1\times10^{9}$  & $34{,}116$
& $2\times10^{9}$  & $37{,}216$
& $4\times10^{7}$  & $13{,}241$
& $1.9\times10^{9}$ & $70{,}660$ \\

Cornell
& $2\times10^{8}$  & $143{,}109$
& $1\times10^{9}$  & $473{,}184$
& $6\times10^{8}$  & $542{,}566$
& $1.9\times10^{9}$ & $506{,}437$ \\

Texas
& $2\times10^{8}$  & $143{,}109$
& $1\times10^{9}$  & $473{,}184$
& $6\times10^{8}$  & $542{,}566$
& $1.9\times10^{9}$ & $506{,}437$ \\

Wisconsin
& $3\times10^{8}$  & $143{,}109$
& $2\times10^{9}$  & $473{,}184$
& $1\times10^{9}$  & $658{,}370$
& $3.0\times10^{9}$ & $506{,}437$ \\

Chameleon
& $2\times10^{9}$  & $182{,}917$
& $2\times10^{10}$ & $632{,}416$
& $1\times10^{10}$ & $2{,}379{,}783$
& $2.7\times10^{10}$ & $632{,}416$ \\

Squirrel
& $9\times10^{9}$  & $167{,}813$
& $1\times10^{11}$ & $572{,}000$
& $7\times10^{10}$ & $4{,}924{,}172$
& $1.3\times10^{11}$ & $605{,}253$ \\

Roman-empire
& $2\times10^{10}$ & $54{,}162$
& $1\times10^{13}$ & $117{,}344$
& $1\times10^{12}$ & $6{,}850{,}778$
& $1.3\times10^{13}$ & $117{,}344$ \\

Amazon-ratings
& $3\times10^{10}$ & $53{,}317$
& $2\times10^{13}$ & $114{,}016$
& $1\times10^{12}$ & $7{,}398{,}933$
& $1.5\times10^{13}$ & $114{,}016$ \\

\bottomrule
\end{tabular}
}
\end{table*}

\paragraph{Asymptotic Complexity}  
We provide an estimate of the asymptotic complexity of our model at inference time. 

\begin{enumerate}
    \item \textbf{Linear Transformation.} The feature transformation is defined as \[\mathbf{X'} = (\mathbf{I}_n \otimes \mathbf{W}_{1}) \mathbf{X} \mathbf{W}_{2}\] where $\mathbf{W}_{1}\in\mathbb{R}^{d\times d}$ and $\mathbf{W}_{2}\in\mathbb{R}^{f\times f}$. The resulting complexity is $\mathcal{O}(n(d^{2}f + df^{2})) = \mathcal{O}(n(cd + cf)) = \mathcal{O}(nc^{2})$, where $c = df$.
    \item \textbf{Message Passing.}  
    Once the Laplacian operator has been assembled, message passing reduces to a sparse-dense 
    matrix multiplication of the form 
    \[
    \mathbf{Q}^{\vec{\mathcal{F}}}_N \mathbf{X'}.
    \] 
    The sparsity pattern of $\mathbf{Q}^{\vec{\mathcal{F}}}_N$ comes directly from the 
    incidence matrix: each hyperedge of size \(|e|\) induces \(|e|^{2}\) nonzero blocks 
    through the outer product \(\mathbf{B}^{(q)}(e, :)^{\dagger} \mathbf{B}^{(q)}(e,:)\). Summing across all hyperedges gives 
    a total of \(\mathsf{S}_2 = \sum_{e \in \mathcal{E}} |e|^{2} = \mathcal{O}\big(m\bar{e}^2)\) nonzero blocks, where $\bar{e}$ is the average hyperedge size\footnote{One could also upper bound the \(\mathsf{S}_2\) term with \(\mathcal{O}\big(m{n}^2)\), however, that approximation would be highly pessimistic, considering a fully-dense representation of the hypergraph, where each hyperedge connects all nodes.}.  
    Applying the Laplacian then requires
    \(
    \mathcal{O}(m \bar{e}^2 c)
    \) for diagonal maps and \(\mathcal{O}(m\bar{e}^2d c)\) with non-diagonal maps,
    
    \item \textbf{Learning the Sheaf.} Restriction maps are predicted as
    \[
        \Phi(\mathbf{x}_v, \mathbf{x}_e) = \sigma\!\left(\mathbf{V}\bigl(\mathbf{x}_v  \,\|\, \mathbf{x}_e\bigr)\right)
    \]
    where \(\mathbf{V}\) is a learnable transformation and \(\sigma\) a nonlinearity. The resulting \(f\)-dimensional vector is 
    then used as input to \(V\) for every node-edge incidence. Indicating with \(\bar{v}\) the average number of partecipations of a node to an hyperedge, the computational complexity is \(\mathcal{O}(\bar{v} m c)\) in the 
    \emph{diagonal} case, and \(\mathcal{O}(\bar{v} m d^{2} c)\) in the 
    \emph{non-diagonal} case.

   \item \textbf{Constructing the Laplacian.} In the hypergraph setting we assemble
   \[
      \mathbf{Q}^{\vec{\mathcal{F}}}_N  = \Dvc^{-\frac12} {\mathbf{B}^{(q)}}^\dagger \mathbf{D}_E^{-1} \mathbf{B}^{(q)} \Dvc^{-\frac12}.
    \]
     
    The work naturally splits into two steps:
    \begin{enumerate}
      \item \emph{Degree normalization.}  
      This involves computing the node and hyperedge degree matrices,
      \(\Dvc^{-\tfrac{1}{2}}\) and \(\mathbf{D}_E^{-1}\). For vertex degree normalization each node requires
      aggregating contributions from its incident hyperedges, giving
      \(\mathcal{O}(m \bar{e}d)\) operations in the diagonal case and
      \(\mathcal{O}(m \bar{e} d^{3})\) in the non-diagonal case (since
      each block is \(d\times d\)), to which it must be added the cost of inverting the block-diagonal matrices, adding to the complexity \(\mathcal{O}\big(nd) \) in the diagonal case and \(\mathcal{O}\big(nd^{3})\) in the non-diagonal case while since \( \mathbf{D}_E^{-1} \) is obtained by expanding to matrix for the scalar hyperedge degrees $\delta_e$ this cost adds a trascurable term to the asymptotic complexity.
      \item \emph{Sparse product.}  
        Forming the term    \[\mathbf{Q}^{\vec{\mathcal{F}}}_N  = \Dvc^{-\frac12} {\mathbf{B}^{(q)}}^\dagger \mathbf{D}_E^{-1} \mathbf{B}^{(q)} \Dvc^{-\frac12}.
        \]
        requires, for each hyperedge \(e\), generating block interactions among all pairs 
        of nodes it contains. This gives a total of \(\mathsf{S}_2=\sum_{e\in\mathcal{E}} |e|^{2}\) 
        block products. The cost is \(\mathcal{O}(\mathsf{S}_2 d)\) in the diagonal case and 
        \(\mathcal{O}(\mathsf{S}_2 d^{3})\) in the non-diagonal case. Since the normalization terms \(\Dvc^{-\tfrac{1}{2}}\) are block-diagonal operations they do not contribute substantially in the overall 
        complexity. Since \( \mathsf{S}_2 = \mathcal{O} \big( m \bar e^{2} \big) \), the dominant cost becomes 
        \(\mathcal{O}(m \bar e^{2} d)\) for diagonal maps and \(\mathcal{O}(m \bar e^{2} d^{3})\) 
        for non-diagonal maps.
    \end{enumerate}
    \end{enumerate}
    By summing the overall contributions we get:
    \(
    \mathcal{O}\!\left(
    n\,(c^{2}+d)
    \;+\;
    m\,(\,\bar e d+\bar e^{2}(d+c)+\bar v\,c\,)
    \right)
    \) in the diagonal case and \(
    \mathcal{O}\!\left(
    n\,(c^{2}+d^{3})
    +
    m\,(\,\bar e d^{3}+\bar e^{2}(d^{3}+dc)+\bar v\,d^{2}c\,)
    \right)
    \) in the non-diagonal case.
    
    \paragraph{Considerations on the Asymptotic Complexity}  
    The leading cost arises from the \emph{Laplacian assembly step}, which scales as 
    $\mathcal{O}(m \bar e^{2} d)$ in the diagonal case and 
    $\mathcal{O}(m \bar e^{2} d^{3})$ in the non-diagonal case. This quadratic dependence on the average hyperedge size $\bar e^{2}$ makes the method 
    particularly sensitive to hypergraphs with densely populated hyperedges. In practice, this means that even when the number of nodes and hyperedges are moderate, the presence of densely populated hyperedges can dominate the computational cost.
    % This limitation clearly emerged in some of our experiments, especially on \texttt{email-EU} and \texttt{Telegram} datasets, where the computational 
    % requirements of Laplacian construction proved to be particularly expensive, despite the relatively modest number of nodes and hyperedges involved.

\subsection{Architectural Choices}

% We provide additional details regarding the implementation of our models, with a particular emphasis on the architectural choices that contribute to their stability and expressiveness. While the main body of the paper already outlines the general structure, in this appendix we describe in greater depth some of the building blocks that were crucial in practice. Specifically, we highlight the role of layer normalization, the integration of non-linear activations, skip connections and introduce the concept of dynamic sheaves. Finally we  provide a more detailed explanation of the workings of \light.

% \begin{itemize}
% \item \textbf{Layer Normalization}: employed to stabilize training dynamics and improve convergence by re-centering and re-scaling intermediate representations.
% \item \textbf{Residual Connection and Activation Functions}: integrated to facilitate gradient flow across layers and to enable non-linear feature transformations while preserving information from previous stages.
% \item \textbf{\light}: a lightweight projection layer used to align features into a common representation space before further processing, which proves particularly important when gradients cannot flow directly through other components of the model.
% \end{itemize}

% \subsection{Normalization, Residual Connections, and Activation Functions}  

\paragraph{Layer Normalization}  
Each layer may optionally include layer normalization, with this choice considered a tunable hyperparameter, since it improves training stability and overall performance. Since the input signal to each convolutional layer is complex-valued, we adopt a complex normalization strategy as proposed in \cite{trabelsi2018deep,barrachina2023theoryimplementationcomplexvaluedneural}, where each complex feature is treated as a two-dimensional real vector $(\Re(x), \Im(x))$. Specifically, we compute the full $2 \times 2$ covariance matrix:  
\[
\Sigma =
\begin{bmatrix}
\sigma_{rr} & \sigma_{ri} \\
\sigma_{ri} & \sigma_{ii}
\end{bmatrix},
\quad
\tilde{\mathbf{x}} = \Sigma^{-\frac{1}{2}} (\mathbf{x} - \boldsymbol{\mu}),
\]
where $\boldsymbol{\mu} = (\mu_r, \mu_i)$ is the mean vector of the real and imaginary parts. The whitening transform $\Sigma^{-\frac{1}{2}}$ ensures that the two components are jointly normalized and decorrelated. To enhance flexibility, we apply an optional learnable affine transformation in the complex plane:  
\[
x_o = \gamma \tilde{x} + \beta,
\]
with trainable parameters $\gamma \in \mathbb{R}^{2 \times 2}$ and $\beta \in \mathbb{R}^2$. These are initialized as $\gamma = \tfrac{1}{\sqrt{2}} I_2$ and $\beta = 0$, thereby preserving the norm of unit-modulus inputs while maintaining the identity mapping at initialization.  

\paragraph{Residual Connections}  
Following observations from~\cite{bodnar2023neuralsheafdiffusiontopological}, we optionally include residual connections in our convolutional layers, which we found to help the architecture in certain datasets. The use of residuals is treated as a tunable hyperparameter (see \cref{app:tuning}). With this addition, a convolutional layer takes the form:  
\[
\mathbf{X}_{t+1} 
= \sigma\left( \mathbf{Q}^{\vec{\mathcal{F}}}_N\,(\mathbf{I}_n \otimes \mathbf{W}_1)\, \mathbf{X}_t\, \mathbf{W}_2 + \mathbf{X}_t \right) \in \mathbb{C}^{nd \times f}.
\]

\paragraph{Activation Function}  
For the activation function, we adopt the complex ReLU commonly employed in related works~\citep{zhang2021magnetneuralnetworkdirected,fiorini2023sigmanetlaplacianrule,fiorini2024let}. It is defined as:
\[
\mathrm{ReLU}(x) = 
\begin{cases}
    x, & \text{if } \Re(x) > 0, \\[2pt]
    0, & \text{otherwise}.
\end{cases}
\]

% \paragraph{Dynamic Sheaves}

% In our implementation, we also explore the possibility of assigning a distinct Laplacian operator to each layer of the network, following the approach introduced in~\cite{bodnar2023neuralsheafdiffusiontopological}. Concretely, this means that each layer can employ its own set of restriction maps to construct a layer-specific Directed Sheaf Hypergraph Laplacian. The motivation behind this design is to allow each layer to specialize in learning signal propagation patterns that are best suited to its depth within the architecture. However, enabling this level of flexibility can significantly increase the computational cost, as it requires computing and differentiating through a separate Laplacian construction at every layer. In particular, the MLPs responsible for predicting the restriction maps must remain differentiable throughout the entire training process to enable proper gradient flow. This choice is treated as a hyperparameter in our experiments, as seen in \cref{app:tuning}.

\paragraph{\light}
%During our experiments, we discovered that detaching the Laplacian construction from the computational graph prevented gradients from flowing back to the MLPs responsible for predicting the restriction maps. As a consequence, these parameters were not properly updated during training. Surprisingly, despite this, the models still achieved strong performance across several datasets. Motivated by this observation, we introduced \light, a simplified variant that deliberately detaches the Laplacian construction, thereby reducing computational overhead.
The architecture of \light is illustrated in \cref{fig:light-visualization}. The model takes as input a node feature matrix \(X_{\text{input}}\), which is projected into a higher-dimensional stalk space via a learnable linear transformation. This representation is then used both in the message-passing pipeline and as input to the MLP that predicts the restriction maps \( \vec{\mathcal{F}}_{v \trianglelefteq e}\). Unlike \approach, the Laplacian operator is built outside the computational graph, so the MLP parameters are not updated during training. Nevertheless, the initial projection layer remains trainable, which allows the model to indirectly influence the restriction maps: by shaping the input embeddings, the network can still control the outputs of the MLP. In this way, even though the restriction map MLPs are frozen, the model is still able to predict good values of embeddings and restriction maps, as confirmed by the empirical results in \cref{tab:results_real_world_complete,tab:results_synthetic}.

\begin{figure}[h]
    \centering
    \includegraphics[width=\linewidth]{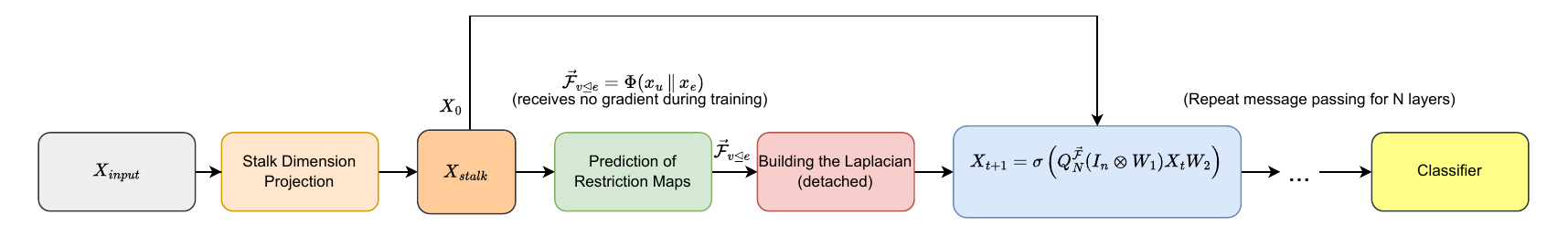}
    \caption{Illustration of the \light architecture. The Laplacian construction is detached from the computational graph, but the initial stalk projection layer remains trainable, allowing the model to indirectly influence the restriction maps.}
    \label{fig:light-visualization}
\end{figure}
\section{Experimental Setup}
\label{app:exp_setup}

\subsection{Hardware Details}
All experiments are carried out on two different workstations: one equipped with two NVIDIA RTX~4090 GPUs (24~GB each) and an AMD Ryzen~9~7950X 6-core processor, and another featuring an Intel Core i9-10940X 14-core CPU (3.3~GHz), 256~GB of RAM, and a single NVIDIA RTX~A6000 GPU with 48~GB of VRAM. We utilized the WandB platform to monitor training procedures and to carry out hyperparameter tuning for each model. 

\subsection{Selected Baselines} \label{app:baselines}

We compare our models against twelve state-of-the-art methods from the hypergraph learning literature. From the \emph{undirected} hypergraph-learning literature we include HGNN~\citep{feng2019hypergraph}, HNHN~\citep{HNHN2020}, 
UniGCNII~\citep{ijcai21-UniGNN}, LEGCN~\citep{LEGCN}, HyperND~\citep{tudisco}, AllDeepSets and AllSetTransformer~\citep{allset}, ED-HNN~\citep{wang2022equivariant}, SheafHyperGNN~\citep{duta} and PhenomNN~\citep{wang2023hypergraphenergyfunctionshypergraph}. From the \emph{directed} hypergraph-learning literature, we consider GeDi-HNN~\citep{fiorini2024let} and DHGNN~\citep{DHGNN} as our baselines. DHGNN was originally designed for link prediction on directed graph datasets and relies on a learnable embedding table to represent node features. In our evaluation, we report the model's performance using both this original embedding approach and an alternative setup with explicit node features. 

% We adapted their implementation of the approximate Laplacian operator as it was originally tailored for directed graphs, despite their definition in principle working also for directed hypergraphs. 

% Among the baselines present in the literature, we also examined DHMConv~\citep{dhmconv}, which is introduced as a spatial convolution for directed hypergraphs. In practice, though, its implementation is designed for directed 2-uniform hypergraphs (i.e. standard graphs). Unlike spectral approaches such as DHGNN, which leverage a Laplacian construction, spatial methods such as DHMConv rely on edge-wise indexing mechanisms that are inherently tied to a graph structure and cannot be meaningfully applied to the directed hypergraph setting considered in our work, where each hyperedge can contain multiple nodes and are not restricted to pair-wise relations.

\subsection{Hyperparameter tuning}\label{app:tuning}
For tuning all the models, we employ a Bayesian optimization method. All models are trained for up to 500 epochs with early stopping set to 200 epochs. We employ Adam~\citep{kingma2017adammethodstochasticoptimization} for optimizing the model parameters with \(\text{lr} \in \{0.02,\,0.01,\,0.005,\,0.001\}\), \(\text{wd}\in\{0,\,5\times10^{-5},5\times10^{-4}\}\). For all the models, we adopt a dropout \(\in\{0.1,0.2,\ldots,0.9\}\), and for each model that has a selectable number of layers for the final classifier we fix it to 2. 
For each baseline, we select a range of parameters consistent with those investigated in their respective original works:

\begin{itemize}
  \item AllDeepSets, ED-HNN: basic blocks \(\{2,4,8\}\); MLPs per block \(\{1,2\}\); MLP hidden width  \(\{64,128,256,512\}\); classifier width \(\{64,128,256\}\).
  \item AllSetTransformer: basic blocks \(\{2,4,8\}\); MLPs per block \(\{1,2\}\); hidden MLP width \(\{64,128,256,512\}\); classifier width \(\{64,128,256\}\); heads \(\{1,4,8\}\).
  \item UniGCNII, HGNN, HNHN, LEGCN: basic blocks \(\{2,4,8\}\); MLP hidden width \(\{64,128,256,512\}\).
  \item HyperND: classifier width \{64,128,256\}.
  \item PhenomNN: basic blocks \(\{2,4,8\}\); hidden width \(\{64,128,256,512\}\); \(\lambda_{0}\in\{0.1,0,1\}\); \(\lambda_{1}\in\{0.1,50,1,20\}\); propagation steps \(\{8,16\}\).
  \item GeDi-HNN: convolutional layers \(\{1,2,3\}\); MLP hidden width \(\{64,128,256,512\}\); classifier width \(\{64,128,256\}\).
  \item DHGNN, DHGNN (w/ emb.),  basic blocks \(\{2,4,8\}\); hidden width \(\{64,128,256,512\}\), classifier width \{64,128,256\}.
  \item SheafHyperGNN, \approach, \light : 
    \begin{itemize}
      \item \(\text{sheaf dropout}\in\{\texttt{false},\texttt{true}\}\)
      \item \(\text{convolutional layers}\in\{1,\ldots,5\}\)
      \item MLP hidden width \(\{64,128,256,512\}\)
      \item classifier width \{64,128,256\}
      \item \(d\in\{1,\ldots,6\}\)
      \item \(\text{sheaf actvation}\in\{\texttt{sigmoid},\texttt{tanh},\texttt{none}\}\)
      \item \(\text{left projection}\in\{\texttt{false},\texttt{true}\}\)
      \item \(\text{residual}\in\{\texttt{false},\texttt{true}\}\)
      \item \(\text{dynamic sheaf}\in\{\texttt{false},\texttt{true}\}\)
      \item \(q\in\{0.00,\,0.05,\,0.10,\,0.15,\,0.20,\,0.25\}\) (for \approach \& \light only)
      % \item \(\text{add identity} \in \{\texttt{false},\texttt{true}\}\) (for \approach \& \light only)
    \end{itemize}
\end{itemize}
\subsection{Datasets Description}\label{app:datasets-description}

We follow the data splits proposed by~\cite{zhang2021magnetneuralnetworkdirected} for the \texttt{Telegram}, \texttt{Texas}, \texttt{Wisconsin},
\texttt{Cornell}.
For \texttt{Chameleon} and \texttt{Squirrel} we adopt the splits proposed by~\cite{platonov2023a}.
For \texttt{Roman-empire} and \texttt{Amazon-Ratings} we adopt the splits proposed by \cite{platonov2023a} and adopt the splits of~\cite{allset} for the remaining ones. 
In all cases, the datasets are partitioned into 50\% training, 25\% validation, and 25\% test samples. 
For the \texttt{email-Enron} and \texttt{email-EU} datasets and for all synthetic datasets, node attributes are not available. In these cases, we resort to structural features, representing each node by its degree. The statistics of the 12 real-world datasets as well as synthetic ones are provided in \cref{tab:dataset_statistics}. 
The datasets used for the experiments are:
\begin{itemize}
    \item  \texttt{Cora}, \texttt{Citeseer}
Standard citation benchmarks in which vertices represent research papers and directed edges encode citation links. Node attributes are constructed from text using bag-of-words representations of the documents.

\item \texttt{email-Enron}, \texttt{email-EU}
A corporate email communication network built from Enron’s message logs. Nodes correspond to email accounts and edges record sender interactions. As ground-truth labels are unavailable, we derive node classes via the Spinglass community detection method~\cite{Reichardt_2006}.

\item \texttt{Texas, Wisconsin, Cornell}
WebKB datasets collected from university computer science departments. Each node is a webpage, hyperlinks are edges, and features are bag-of-words over page content. Pages are annotated into five categories: student, project, course, staff, and faculty.

\item \texttt{Telegram}
An interaction network extracted from Telegram, capturing exchanges among users who propagate political content.

\item \texttt{Squirrel}, \texttt{Chameleon}
The Squirrel and Chameleon datasets consist of articles from the English Wikipedia (December 2018). Nodes represent articles, and edges represent mutual links between them. Node features indicate the presence of specific nouns in the articles. Nodes are grouped into five categories based on the original regression targets.

\item \texttt{Roman-empire} The dataset is based on the \emph{Roman Empire} article from English Wikipedia, which was selected since it is one of the longest articles on Wikipedia and it follows the construction proposed by \cite{platonov2023a}. Each node in the graph corresponds to one (non-unique) word in the text.

\item \texttt{Amazon-ratings}
The dataset, as proposed by \cite{platonov2023a}, is based on the Amazon product co-purchasing network metadata dataset from SNAP Datasets~\cite{leskovec2014snap}. Nodes are products (books, music CDs, DVDs, VHS video tapes), and edges connect products that are frequently bought together. 

\item \texttt{Synthetic}
Introduced in \cite{fiorini2024let} by following the methodology adopted in~\cite{zhang2021magnetneuralnetworkdirected}, these datasets are built as follows: a vertex set \(V\) is partitioned into \(c\) equally sized classes \(C_1,\ldots,C_c\). For each class \(C_i\), we sample \(I_i\) \emph{intra-class} hyperedges that are undirected. The cardinality of each hyperedge is drawn uniformly from \(\{h_{\min},\ldots,h_{\max}\}\), and its nodes are sampled uniformly from \(C_i\). For each ordered pair of distinct classes \((C_i,C_j)\) with \(i<j\), we create \(I_o\) \emph{inter-class directed} hyperedges. For every such hyperedge \(e\), the tail set \(T(e)\) is sampled from \(C_i\) and the head set \(H(e)\) from \(C_j\); the sizes \(|T(e)|\) and \(|H(e)|\) are drawn uniformly from \(\{h_{\min},\ldots,h_{\max}\}\).
% This induces a directional flow from \(C_i\) to \(C_j\) only when \(i<j\).
% Using this procedure, we generate three datasets with \(n=500\) nodes, \(c=5\) classes, \(h_{\min}=3\), \(h_{\max}=10\), \(I_i=30\) intra-class hyperedges per class, and an increasing number of inter-class directed hyperedges \(I_o \in \{10,30,50\}\).

\end{itemize}

\begin{table*}[h]
\centering
\caption{Statistics of the datasets used in our experiments. Reported are the number of nodes, features, hyperedges, and classes, as well as the average hyperedge size ($\lvert e \rvert$), the average node degree ($\lvert v \rvert$), and the clique-expansion (CE) homophily computed as in \cite{wang2022equivariant}.}
\label{tab:dataset_statistics}
\resizebox{\textwidth}{!}{%
\begin{tabular}{lrrrrrrr}
\toprule
\textbf{Dataset} & \textbf{\# Nodes} & \textbf{\# Features} & \textbf{\# Hyperedges} & \textbf{\# Classes} & \textbf{avg $|e|$} & \textbf{avg $|v|$} & \textbf{CE homophily} \\
\midrule
Roman-empire      & 22,662 & 300    & 22,662 & 18 & 2.73  & 2.73  & 0.2363 \\
Squirrel    & 2,223  & 2,089  & 2,060  & 5  & 23.81 & 22.07 & 0.2448 \\
email-EU    & 986    & --     & 787    & 10 & 43.36 & 34.61 & 0.2608 \\
Telegram    & 245    & 1      & 183    & 4  & 49.70 & 37.12 & 0.2854 \\
Chameleon   & 890    & 2,325  & 797    & 5  & 12.11 & 10.84 & 0.3221 \\
email-Enron & 143    & --     & 139    & 7  & 19.58 & 19.03 & 0.3251 \\
Cornell     & 183    & 1,703  & 96     & 5  & 4.07  & 2.14  & 0.4200 \\
Wisconsin   & 251    & 1,703  & 170    & 5  & 3.94  & 2.67  & 0.4398 \\
Amazon-ratings      & 24,492 & 300    & 24,456 & 5  & 5.63  & 5.62  & 0.4460 \\
Texas       & 183    & 1,703  & 110   & 5  & 3.81 & 2.29 & 0.5049 \\
Citeseer    & 3,312  & 3,703  & 1,951  & 6  & 3.35 & 1.98  & 0.7947 \\
Cora        & 2,708  & 1,433  & 1,565  & 7  & 4.47  & 2.58  & 0.8035 \\
\midrule
$I_o=10$ & 500 & -- & 250 & 5 & 9.05  & 4.53  & 0.6233 \\
$I_o=30$ & 500 & -- & 450 & 5 & 10.79 & 9.71  & 0.5020 \\
$I_o=50$ & 500 & -- & 650 & 5 & 11.63 & 15.12 & 0.4528 \\
\bottomrule
\end{tabular}
}
\end{table*}

\subsection{Directed Hypergraph from a Directed Graph}\label{app:transformation}

Given a directed graph \(G=(V,E)\), let the out-neighborhood of \(v\in V\) be
\[
N_{\mathrm{out}}(v)=\{\,w\in V \mid (v,w)\in E\,\}.
\]
We build a directed hypergraph \(\mathcal{H}=(V,\mathcal{E})\) by creating one hyperedge \(e_v\) for each node with its outgoing edges and setting
\[
T(e_v)=\{v\},\qquad H(e_v)=N_{\mathrm{out}}(v).
\]
Thus every hyperedge has a tail consisting of a single node and a head set containing all nodes belonging to the neighborhood of that tail. A clear example of this construction procedure can be visualized in \cref{fig:transformation}.

\begin{figure}[h]
    \centering
    \includegraphics[width=0.8\linewidth]{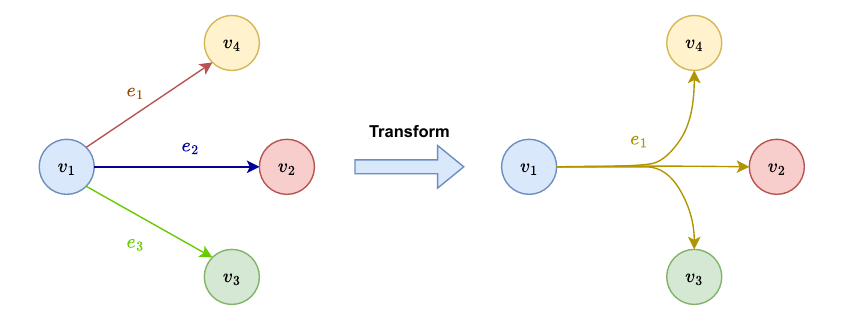}
    \caption{Example of the creation of a directed hyperedge from the out-neighborhood of a node. 
    Suppose we have a graph where node $v_1$ connects to nodes $v_2$, $v_3$, and $v_4$, so that $(v_1,v_2)$, $(v_1,v_3)$, and $(v_1,v_4)$ belong to $E$. The construction procedure yields a directed hyperedge $e_{1}$ with tail set $T(e_{1}) = \{v_1\}$ and head set $H(e_{1}) = \{v_2, v_3, v_4\}$.}
    \label{fig:transformation}
\end{figure}

This formulation preserves the source-target semantics of the original graph by expressing them as a higher-order relation. 
Such hyperedges are often referred to as \emph{forward directed} hyperedges \citep{GALLO1993177}. When every hyperedge is forward directed, the structure is a \emph{forward directed hypergraph}, which is the case for all real-world datasets considered in this work.

\subsection{Hyperedge Classification for Molecular Reaction Type Prediction}\label{app:hyperedge}

In \cref{tab:dataset_molecular_1} and \cref{tab:dataset_molecular_2} we show the distribution of labels for the directed molecular reaction prediction datasets employed for the hyperedge classification task. As mentioned in \cref{app:extended_results}, we evaluate all models on additional real-world molecular prediction datasets (see \cref{tab:f1_scores}) from the perspective of a hyperedge classification task.
To do so, before feeding the output of the last convolutional layer to the classifier, we perform an aggregation (sum) of all node representations belonging to a given hyperedge. Specifically, if $\mathbf{X}_{\text{node}}$ is the final feature matrix of shape $N \times F$, where $N$ is the number of nodes and $F$ is the feature dimension, we can compute hyperedge-level representations using a (real-valued and binary) incidence matrix $\mathbf{H}$ of shape $N \times E$ as follows:
\[
\mathbf{H}^\top \mathbf{X}_{\text{node}} = \mathbf{X}_{\text{edge}} \in \mathbb{R}^{E \times F}.
\]
The resulting matrix $\mathbf{X}_{\text{edge}}$ is then fed to the classifier which will output a matrix of shape $E \times C$, where $C$ is the number of classes.

\begin{table}[h]
\centering
\begin{tabular}{c c c}
\toprule
\textbf{Class} & \textbf{\# Hyperedges} & \textbf{Percentage (\%)} \\
\midrule
0 & 15{,}151 & 30.29 \\
1 & 11{,}896 & 23.78 \\
2 & 5{,}662  & 11.32 \\
3 & 909      & 1.82 \\
4 & 672      & 1.34 \\
5 & 8{,}237  & 16.47 \\
6 & 4{,}614  & 9.23 \\
7 & 811      & 1.62 \\
8 & 1{,}834  & 3.67 \\
9 & 230      & 0.46 \\
\bottomrule
\end{tabular}
\caption{Label distribution for the Molecular-1 dataset.}
\label{tab:dataset_molecular_1}
\end{table}
\begin{table}[h]
\centering
\begin{tabular}{c c c}
\toprule
\textbf{Class} & \textbf{\# Hyperedges} & \textbf{Percentage (\%)} \\
\midrule
0 & 960   & 31.78 \\
1 & 1{,}536 & 50.84 \\
2 & 213   & 7.05 \\
3 & 54    & 1.79 \\
4 & 226   & 7.48 \\
5 & 32    & 1.06 \\
\bottomrule
\end{tabular}
\caption{Label distribution for the Molecular-2 dataset.}
\label{tab:dataset_molecular_2}
\end{table}
\section{On previous proposals of the Sheaf Hypergraph Laplacian} \label{app:iulia-sheaf}

In this section, we revisit the definition of the Sheaf Hypergraph Laplacian proposed in~\cite{duta}, noting that it fails to satisfy basic spectral properties expected of a Laplacian operator, most notably positive semidefiniteness.
This shortcoming motivates our formulation, which, as discussed in \cref{sec:generalization-properties}, constitutes (to our knowledge) the first definition of a Sheaf Hypergraph Laplacian that is fully consistent with the spectral requirements of a convolutional operator also in the undirected setting. For comparison, we recall the (called linear in the paper---the nonlinear one is, in essence, the Laplacian of a 2-uniform hypergraph) Laplacian of \cite{duta}.
\begin{definition} \label{def:iulia-laplacian}
Let $\mathcal{H}=(V,E)$ be a hypergraph with hyperedge degrees $\delta_e$ and let $\mathcal{F}_{v \trianglelefteq e} : \mathbb{R}^d \to \mathbb{R}^d$ be linear restriction maps from node $v$ to hyperedge $e$.
The Laplacian $\mathbf{L}^{\mathcal{F}} \in \mathbb{R}^{nd \times nd}$ has $d\times d$ blocks indexed by $u,v\in V$:
\[
(\mathbf{L}^{\mathcal{F}})_{uu}
= \sum_{e:\,u\in e}\frac{1}{\delta_e}\,
\mathcal{F}_{u \trianglelefteq e}^{\top}\mathcal{F}_{u \trianglelefteq e},
\qquad
(\mathbf{L}^{\mathcal{F}})_{uv}
= -\!\!\!\sum_{\begin{subarray}{c} \\ e: u,v\in e \\  v \neq u\end{subarray}}
\frac{1}{\delta_e}\,
\mathcal{F}_{u \trianglelefteq e}^{\top}\mathcal{F}_{v \trianglelefteq e}
\]
\end{definition}

\cref{def:iulia-laplacian} essentially coincides with a Signless Hypergraph Laplacian, except for the fact that the off-diagonal entries are flipped from positive to negative.\footnote{This is consistent with their implementation.}
%
%This ad hoc modification underlies the flaw in their approach, as the signs on the off-diagonal terms do not arise naturally from the Laplacian construction but are imposed artificially.
%
Such a sign-flip suffices to build a positive semidefinite Laplacian matrix exclusively in the 2-uniform case, where the Laplacian operator for a graph can be obtained by assigning an arbitrary orientation to each edge.
Notice that, in the undirected case, our Laplacian differs from theirs due to featuring a coefficient of $(1-\frac{1}{\delta_e})$ in the diagonal term, rather than $\frac{1}{\delta_e}$.
Considering the proposed definition, we can compute the equation of the Laplacian seen as a linear operator for a signal $x \in \mathbb{R}^{nd}$ as follows:
\begin{align*}\label{eq:disagreement-iulia}
\bigl(\mathbf{L}^{\mathcal{F}}(\mathbf{x})\bigr)_u 
&= \sum_{v \in V} (\mathbf{L}^{\mathcal{F}})_{uv} \, \mathbf{x}_v \\ 
\nonumber &= \sum_{e:\,u \in e} 
\frac{1}{\delta_e}\mathcal{F}_{u \trianglelefteq e}^{\top} \mathcal{F}_{u \trianglelefteq e} \, \mathbf{x}_u
-
\sum_{e:\,u \in e} \sum_{\substack{v \in e \\ v \neq u}} 
\frac{1}{\delta_e}\,
\mathcal{F}_{u \trianglelefteq e}^{\top} 
\mathcal{F}_{v \trianglelefteq e}\, \mathbf{x}_v \\ 
\nonumber &= \sum_{e:\,u \in e}
\frac{1}{\delta_e}\,
\left(
\mathcal{F}_{u \trianglelefteq e}^{\top} 
\mathcal{F}_{u \trianglelefteq e}\, \mathbf{x}_u
-
\sum_{\substack{v \in e \\ v \neq u}} 
\mathcal{F}_{u \trianglelefteq e}^{\top} 
\mathcal{F}_{v \trianglelefteq e}\, \mathbf{x}_v
\right) \\ 
\nonumber  &= \sum_{e:\,u \in e}
\frac{1}{\delta_e}\,
\mathcal{F}_{u \trianglelefteq e}^{\top}
\left(
\mathcal{F}_{u \trianglelefteq e}\, \mathbf{x}_u
-
\sum_{\substack{v \in e \\ v \neq u}} 
\mathcal{F}_{v \trianglelefteq e}\, \mathbf{x}_v
\right).
\end{align*}

Which substantially differs from the expression reported in their respective work, which reads:
$$
\bigl(\mathbf{L}^{\mathcal{F}}(\mathbf{x})\bigr)_u  =
\sum_{e:\,u \in e}
\frac{1}{\delta_e}\,
\mathcal{F}_{u \trianglelefteq e}^{\top}
\sum_{\substack{v \in e \\ v \neq u}} 
\left(
\mathcal{F}_{u \trianglelefteq e}\, \mathbf{x}_u
-
\mathcal{F}_{v \trianglelefteq e}\, \mathbf{x}_v
\right).
$$
Crucially, the latter is the expression that is obtained with our operator in the undirected case, as reported in \cref{eq:disagreement-new}.

Let us illustrate the issue with a numerical example.
Let us consider a hypergraph with node set 
$V=\{v_1,v_2,v_3,v_4\}$ and  $E = \{e_1, e_2\}$ with hyperedges
$e_1=\{v_1,v_2,v_3\}$, $e_2=\{v_2,v_3,v_4\}$, in the case of a \emph{trivial} Sheaf (i.e. $\mathcal{F}_{v\trianglelefteq e} = 1$).
Let $\delta_{e}$ denote the hyperedge size and let
$\mathcal{F}_{u \trianglelefteq e}\in\mathbb{R}$ be the (scalar) restriction on incidence $(u,e)$.

By \cref{def:iulia-laplacian},
the entries of the Laplacian are:
\[
\begin{aligned}
(\mathbf{L}^{\mathcal{F}})_{v_1 v_1} &= \tfrac{1}{\delta_{e_1}}\,
\mathcal{F}_{v_1 \trianglelefteq e_1}^{\top}\mathcal{F}_{v_1 \trianglelefteq e_1}, \\[6pt]
(\mathbf{L}^{\mathcal{F}})_{v_2 v_2} &= \tfrac{1}{\delta_{e_1}}\,
\mathcal{F}_{v_2 \trianglelefteq e_1}^{\top}\mathcal{F}_{v_2 \trianglelefteq e_1}
+ \tfrac{1}{\delta_{e_2}}\,
\mathcal{F}_{v_2 \trianglelefteq e_2}^{\top}\mathcal{F}_{v_2 \trianglelefteq e_2}, \\[6pt]
(\mathbf{L}^{\mathcal{F}})_{v_3 v_3} &= \tfrac{1}{\delta_{e_1}}\,
\mathcal{F}_{v_3 \trianglelefteq e_1}^{\top}\mathcal{F}_{v_3 \trianglelefteq e_1}
+ \tfrac{1}{\delta_{e_2}}\,
\mathcal{F}_{v_3 \trianglelefteq e_2}^{\top}\mathcal{F}_{v_3 \trianglelefteq e_2}, \\[6pt]
(\mathbf{L}^{\mathcal{F}})_{v_4 v_4} &= \tfrac{1}{\delta_{e_2}}\,
\mathcal{F}_{v_4 \trianglelefteq e_2}^{\top}\mathcal{F}_{v_4 \trianglelefteq e_2}, \\[12pt]
(\mathbf{L}^{\mathcal{F}})_{v_1 v_2} &= -\tfrac{1}{\delta_{e_1}}\,
\mathcal{F}_{v_1 \trianglelefteq e_1}^{\top}\mathcal{F}_{v_2 \trianglelefteq e_1}, \\[6pt]
(\mathbf{L}^{\mathcal{F}})_{v_1 v_3} &= -\tfrac{1}{\delta_{e_1}}\,
\mathcal{F}_{v_1 \trianglelefteq e_1}^{\top}\mathcal{F}_{v_3 \trianglelefteq e_1}, \\[6pt]
(\mathbf{L}^{\mathcal{F}})_{v_1 v_4} &= 0, \\[6pt]
(\mathbf{L}^{\mathcal{F}})_{v_2 v_3} &= -\tfrac{1}{\delta_{e_1}}\,
\mathcal{F}_{v_2 \trianglelefteq e_1}^{\top}\mathcal{F}_{v_3 \trianglelefteq e_1}
-\tfrac{1}{\delta_{e_2}}\,
\mathcal{F}_{v_2 \trianglelefteq e_2}^{\top}\mathcal{F}_{v_3 \trianglelefteq e_2}, \\[6pt]
(\mathbf{L}^{\mathcal{F}})_{v_2 v_4} &= -\tfrac{1}{\delta_{e_2}}\,
\mathcal{F}_{v_2 \trianglelefteq e_2}^{\top}\mathcal{F}_{v_4 \trianglelefteq e_2}, \\[6pt]
(\mathbf{L}^{\mathcal{F}})_{v_3 v_4} &= -\tfrac{1}{\delta_{e_2}}\,
\mathcal{F}_{v_3 \trianglelefteq e_2}^{\top}\mathcal{F}_{v_4 \trianglelefteq e_2}.
\end{aligned}
\]

Numerically, we have:
\[
\begin{bmatrix}
\tfrac{1}{3} & -\tfrac{1}{3} & -\tfrac{1}{3} & 0 \\[6pt]
-\tfrac{1}{3} & \tfrac{2}{3} & -\tfrac{2}{3} & -\tfrac{1}{3} \\[6pt]
-\tfrac{1}{3} & -\tfrac{2}{3} & \tfrac{2}{3} & -\tfrac{1}{3} \\[6pt]
0 & -\tfrac{1}{3} & -\tfrac{1}{3} & \tfrac{1}{3}
\end{bmatrix}.
\]
The spectrum of the above Laplacian is:
\[
\operatorname{eig}\left(\mathbf{L^\mathcal{F}}\right)
= \left\{\tfrac{4}{3},\ \tfrac{1}{3},\ \tfrac{1+\sqrt{17}}{6},\ \tfrac{1-\sqrt{17}}{6}\right\}.
\]

Since a negative eigenvalue appears, \(\mathbf{L}^{\mathcal{F}}\) is \emph{not} positive semidefinite in this example.

\end{document}